\documentclass[journal]{IEEEtran}
\usepackage{graphicx}
\usepackage[ruled,vlined, linesnumbered]{algorithm2e}
\usepackage{cite}
\usepackage{amsfonts,amssymb,amsmath,amsthm}
\usepackage{enumerate}
\usepackage{color}
\usepackage{balance}
\newtheorem{theorem}{Theorem}
\newtheorem{remark}{Remark}
\newtheorem{lemma}{Lemma}
\newtheorem{assumption}{Assumption}

\newtheorem{problem}{Problem}
\usepackage[english]{babel}
\usepackage{url}
\newcommand{\Abf}{\pmb{A}}\newcommand{\Bbf}{\pmb{B}}\newcommand{\Cbf}{\pmb{C}}\newcommand{\Hbf}{\pmb{H}}\newcommand{\Ibf}{\pmb{I}}\newcommand{\Lbf}{\pmb{L}}\newcommand{\Mbf}{\pmb{M}}\newcommand{\Ubf}{\pmb{U}}\newcommand{\Xbf}{\pmb{X}}
\newcommand{\cbf}{\pmb{c}}\newcommand{\ebf}{\pmb{e}}\newcommand{\fbf}{\pmb{f}}\newcommand{\gbf}{\pmb{g}}\newcommand{\nbf}{\pmb{n}}\newcommand{\pbf}{\pmb{p}}\newcommand{\sbf}{\pmb{s}}\newcommand{\ubf}{\pmb{u}}\newcommand{\vbf}{\pmb{v}}\newcommand{\xbf}{\pmb{x}}
\ifCLASSINFOpdf
\else
\fi
\begin{document}
	%
	\title{Vision Based Autonomous UAV Plane Estimation And Following for Building Inspection}
	%
	%
	%
	
	\author{Yang~Lyu,~\IEEEmembership{Member,~IEEE,}
		Muqing~Cao,~
		Shenghai~Yuan, 
		and~Lihua~Xie,~\IEEEmembership{Fellow,~IEEE}
		\thanks{ 
			Yang Lyu, Muqing Cao, Shenghai Yuan and Lihua Xie (corresponding author) are with School of Electrical and Electronic Engineering, Nanyang Technological University, 50 Nanyang Avenue, Singapore 639798. (email: \{lyu.yang, mqcao, shyuan, elhxie\}@ntu.edu.sg)}
		
		}

	\maketitle
	
	\begin{abstract}
		Unmanned Aerial Vehicle (UAV) has already demonstrated its potential in many civilian applications, and the fa\c{c}ade inspection is among the most promising ones. In this paper, we focus on enabling the autonomous perception and control of a small UAV for a fa\c{c}ade inspection task.
		Specifically, we consider the perception as a planar object pose estimation problem by simplifying the building structure as concatenation of planes, and the control as an optimal reference tracking control problem.
		First, a vision based adaptive observer is proposed which can realize stable plane pose estimation under very mild observation conditions. Second, a model predictive controller is designed to achieve stable tracking and smooth transition in a multiple planes scenario, while the persistent excitation (PE) condition of the observer and the maneuver constraints of the UAV are satisfied. The proposed autonomous plane pose estimation and plane tracking methods are tested in both simulation and practical building fas\c{c}ade inspection scenarios, which demonstrate their effectiveness and practicability.  
	\end{abstract}
	

	%
	\IEEEpeerreviewmaketitle

	\section{Introduction}
	Autonomous vehicle, especially the Unmanned Aerial Vehicle (UAV)  has attracted tremendous research interests in recent years due to its potential capability of improving efficiency and safety in many military and civilian applications \cite{cai2014survey}. Among them, the building {\it fa\c{c}ade inspection} is considered as one of the most promising applications for small size UAVs \cite{ROCA2013128}. During a {fa\c{c}ade inspection} task,  a  UAV  has  to  maintain a  desired  distance to the building fa\c{c}ades while following an inspection path allows  its  on-board inspection sensors  to  cover  the  entire surface area of the building. To carry out the challenging autonomous inspection task, the UAV needs to persistently measure the unknown building structure and continuously control the vehicle so as to effectively cover the building fas\c{c}ades. 
	In this paper, we focus on solving the fa\c{c}ade plane pose estimation and optimal plane following control together, thus demonstrating autonomous and efficient fas\c{c}ade inspection capability.
	
	The estimation of the fas\c{c}ade pose is a prerequisite for autonomous inspection of a subject whose detailed structural information is unknown. The inspection subject, e.g., a building, often has significant structural features, such as points, lines and planes. With proper clustering techniques \cite{peng2019deep,peng2018structured}, most buildings can be approximated as multiple concatenated planes, therefore makes the plane one of the most ubiquitous geometric primitives in an inspection scenario. There are multiple sensing options for the plane estimation, such as camera \cite{Zhou2020Ground}, RGB-D camera \cite{Lee2012Indoor}, LiDAR \cite{asvadi20163d}, and sonar sensor \cite{englot2013three}. The LiDAR sensor is often considered with unparalleled sensing capabilities for robot 3-D space operations. Nevertheless, its high payload requirement may hinder its practical implementation on small UAV with limited payload capabilities. On the other hand, the camera, with its low payload requirement and high fidelity sensing information, is considered as the most common sensor for robot operation. By taking 2D image information from multiple views, the 3D geometry of object can be retrieved\cite{hu2018local,hu2019multi}.
	Therefore, besides serving as an inspection sensor of the building fas\c{c}de, we consider to  implement the camera on-board a moving platform to estimate the pose of the planar fas\c{c}ades. 
	
	The research on vision based structure estimation has been active for years. Methods may be roughly divided into two streams, namely the homography based methods \cite{chavez2017homography,Zhou2020Ground}, and direct depth estimation methods \cite{Spica2013a,Spica2014Active,Giordano2014An,Tahri2017Brunovsky}.  The first method exploits the homography constraint to recover the spatial structure from a moving camera observing feature points. The homography matrix is calculated from feature correspondences of 2D image pairs. As illustrated in \cite{chavez2017homography}, an extended Kalman filter is implemented to estimate the landing plane, with vision measurements derived from the homography relationship. A more recent result reported in \cite{Zhou2020Ground} utilizes a 2D homography matrix to recover the plane scale information from the ground plane and camera with constant height. As the homography based method depends on the feature correspondences, which may require further treatment of features variation and therefore hinder its robustness in a dynamic scenario. The second method
	\cite{Spica2013a,Spica2014Active,Giordano2014An,Tahri2017Brunovsky,huang2019structure} is to obtain the structural information, such as the depth map, based on the structure from motion (SfM) \cite{ullman1979interpretation}. Among them, the active vision scheme  proposed in \cite{Spica2013a} is a well established method for depth estimation in 3-D scenarios. Subsequent works \cite{Spica2014Active} and \cite{Giordano2014An} extended the proposed framework to more structural scenarios, such as plane, sphere, and cylinder. The framework was proved to be locally exponentially stable under the persistent excitation condition that the sensor has a non-zero velocity. However it may still suffer from coupling and nonlinearity. To overcome those drawbacks, \cite{Tahri2017Brunovsky} proposed an improvement by decoupling the model into a linear form. In general the second method is more applicable than the first method during a mobile robot operation as it applies frame by frame feature tracking and is robust to feature gain/loss.
{ 	More recently, the learning based method becomes a new trend by training a deep neural network to infer depth from images \cite{liu2010single,Goldman2019,Poggi2018Towards}. Nevertheless, the learning based method may still not be ready for field applications due to its high dependency on training data and high computation cost.}
	 Therefore in this paper we consider implementing an observer based on SfM to realize stable plane pose estimation during an inspection task. 
	
	After measuring the plane pose, the plane path following for the inspection needs to be designed accordingly. With known inspection subject information, the path design is nothing more than the most commonly researched coverage path planing\cite{galceran2013survey,englot2013three,triangulation2020An}. A detailed review of the coverage path planing with a known inspection map is given in \cite{galceran2013survey}, where different modeling patterns of coverage path planing are detailed. For instance, \cite{englot2013three} proposed a space sampling based method for probabilistic complete path planing, where the availability of a 3-D CAD model is assumed. \cite{triangulation2020An} described a coverage path planing method by dividing the known target region into several regular triangulations and then designing a path satisfying different constraints, such as sensing range, map coverage, collision avoidance. 
	A more realistic assumption is that, although a global map is available, the inspection platform still needs to properly adjust its trajectory based on the measurement from the on-board sensors due to unforeseen situations, such as map inaccuracy, obstacle avoidance, etc. An  on-line  path  re-planing  algorithm  for underwater  structural  inspection using  autonomous  robot  is proposed in \cite{englot2013three}.
	With a given initial nominal path based on a priori information, the re-planing is to continuously adjust the nominal path based on  stochastic trajectory optimization and measurements from an on-board sonar sensor. Similar frameworks are also presented in \cite{Yazici2014Dynamic,Alexis2015Uniform}. 

	In the more challenging scenario without an accurate structure map, the vehicle is controlled based on local observation feedbacks to properly follow the fas\c{c}des of an inspection subject. The path following control of a UAV w.r.t. to planar surfaces is studied in \cite{Herisse2010general}, where the planar surfaces are detected based on cameras and a PID controller is designed with optical flow measurement feedbacks. Further, a receding horizon control law is proposed in \cite{hamada2018receding} that enables a fixed wing UAV to follow the ground plane, with a given lateral trajectory. More recently, a nonlinear Model Predictive Control (MPC) method is implemented in \cite{prediction2019cai} to solve the terrain following problem for an underwater vehicle during a seabed survey task. 
	Inspired by the works aforementioned, we aim at designing an on-line path planing method based on the MPC framework, which is known for its optimality and constraints satisfaction, to realize the plane following while comply with the fas\c{c}ade inspection requirement.

	In this paper, we focus on solving the autonomous fas\c{c}ade inspection with regard to an unknown inspection subject. On-line solutions for both vision based fas\c{c}ade plane pose estimation and optimal inspection control are provided. First, the building fas\c{c}ades are approximately modeled as concatenation of planes and a vision based plane pose observer is proposed. Upon that, a constrained MPC based plane following control strategy which incorporates inspection preferences is designed.
	The contributions of the paper are two folds. First, we propose a vision based observer to directly estimate the normal vector and depth of a plane from 2D image features of consecutive frames. The PE condition that guarantee the asymptotic stability is investigated. Second, we design a MPC based plane following controller that bears multiple constraints, including PE condition and state and control constraints. The feasibility and stability are guaranteed by designing an appropriate adaptive reference update law and terminal conditions. The proposed framework is validated in practical inspection scenarios.
	
	The remaining of the paper is organized as follows: Section \ref{Sec_II} models the inspection problem as plane estimation and path following problems for UAV, and Section \ref{Sec_III} and Section \ref{Sec_IV} respectively describe the plane pose estimation and plane following control. Section V and Section VI describe the simulation and experimental results based on the proposed methods respectively. Section VII concludes the paper.
	
	\section{Problem Formulation}
	\label{Sec_II}
	The fa\c{c}ade inspection problem using a UAV platform is formulated in this section. Specifically, the inspection task is to simultaneously estimate and follow a planar object based on an on-board vision sensor. 
	
	The process of the building inspection task and the coordinates are illustrated in Fig.\ref{fig:Coordination_illustration}. First, the plane pose parameters are estimated in the camera frame, denoted as $\mathcal{C}$, based on the 2D measurements from the image frame, denoted as $\mathcal{I}$. Then the plane parameters are transformed to the global frame, denoted as $\mathcal{G}$, where the inspection preferences are described and the plane following controller is executed.
	
	\begin{figure}{H}
		\centering
		\includegraphics[width=0.9\linewidth]{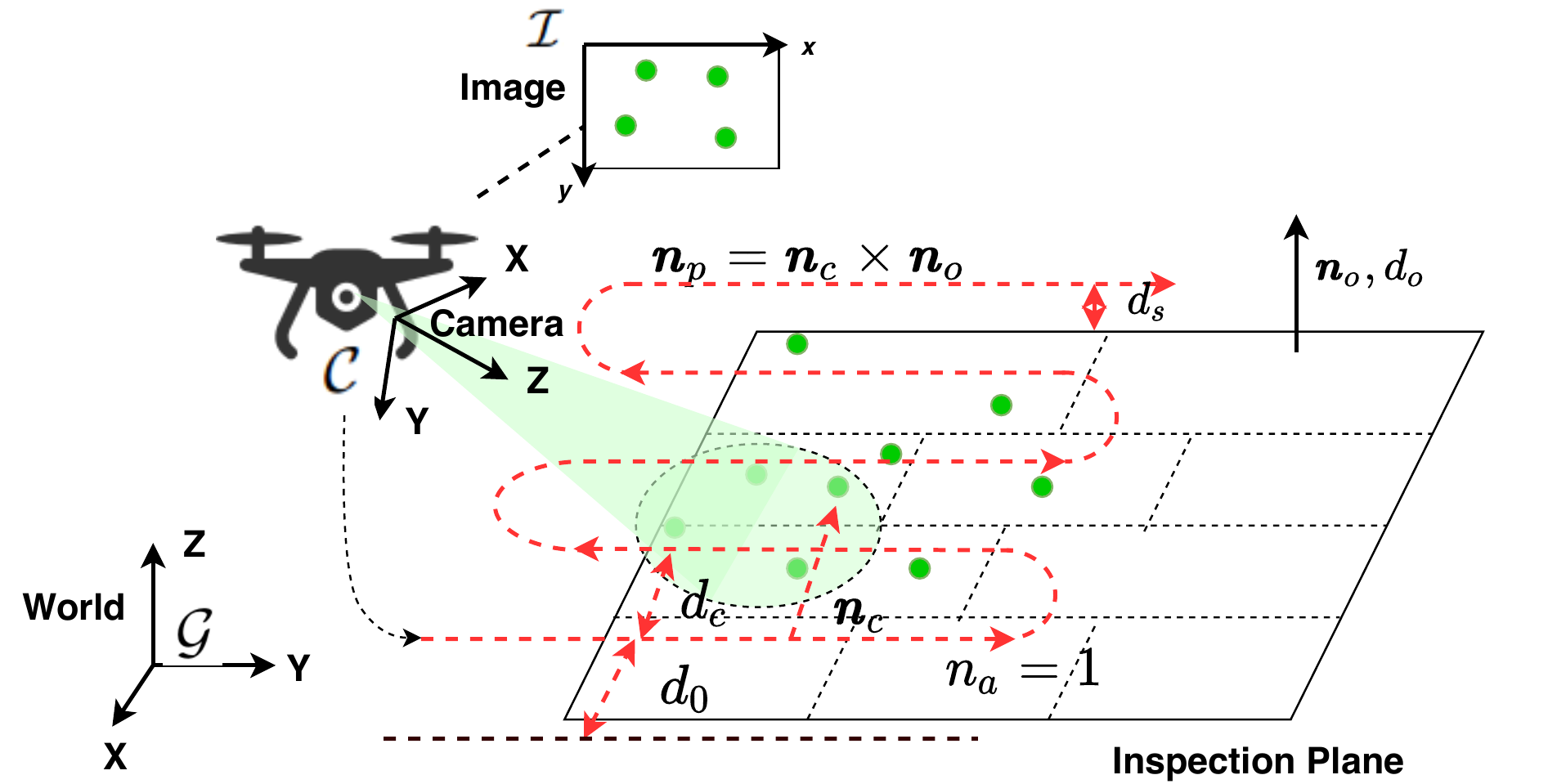}
		\caption{The coordination system illustration and inspection preferences.}
		\label{fig:Coordination_illustration}
	\end{figure}
	\subsection{Models}\subsubsection{Planar object}
	Let $\mathcal{O} = \{\nbf_o, d_o\}$ denote a plane. A point $\pbf$ that belongs to $\mathcal{O}$ satisfies 
	\begin{equation*}
	\nbf_{o}^T\pbf + d_o = 0,
	\end{equation*}
	where $\nbf_o= [n_x, n_y, n_z]^T\in \mathbb{S}^2$ is the plane unit normal vector and $d_o\in \mathbb{R}$ is the distance of $\mathcal{O}$ to the original point of a specific coordinate system.	
	Further in the homogeneous coordinate, the above equation can be written as $\bar{\nbf}_o\bar{\pbf} = 0$, where $\bar{\pbf}=\left[\pbf^T ,1\right]^T$, $\bar\nbf_o = [\nbf_o^T, d_o]^T$. The $\bar\nbf_o$ is not known a prior and needs to be estimated during the inspection operation.
	
	\subsubsection{Moving platform}
	For simplification, the platform is modeled as a discrete-time second order dynamics system in 3-D space: 
	\begin{equation}
	\label{dyn_model}
	\xbf_{k+1} = \Abf \xbf_k + \Bbf \ubf_k,
	\end{equation}
	where
	\[{ \Abf = \begin{bmatrix}
	1 &T_s\\
	0 &1
	\end{bmatrix}\otimes \Ibf_3, \Bbf = \begin{bmatrix}
	0.5T_s^2 \\ T_s
	\end{bmatrix}\otimes \Ibf_3,
}\]
	and the state vector $\xbf_k= [\pbf^T_k, \vbf^T_k]^T\in \mathbb{R}^6$ includes the position and velocity of the platform at time instance $k\in \mathbb{Z}^{+}$ in the global frame $\mathcal{G}$. { $T_s$ is the sampling time interval between time instance. $\ubf_k$ is the acceleration control input.}
	Norm constraints are imposed on the velocity and acceleration control, as
	\[\|\vbf_k\|_b\le \bar{v}, \quad \|\ubf_k\|_b\le \bar{u},\]
	where $\bar{v}$ and $\bar{u}$ represent the maximum velocity and acceleration input respectively.
{ 	$\|\cdot\|_b$ denotes some norms according to the platform property, with $b= 2$ and $b= \infty $ denoting the 2-norm and infinity norm, respectively.}
{ \begin{remark}
	In this paper, we only consider the mission level control of a UAV to achieve the desired inspection trajectory pattern by feeding in acceleration control commands to produce position and velocity outcome, like the linear model (1). 
	We assume that there is a low level autopilot controller on-board the UAV that takes account the nonlinear and coupled dynamic of the platform, and manipulates the physical actuators to achieve the effect of the desired input $\ubf$.
\end{remark}}
	\subsection{Objective}
	\label{objective}
	The UAV based fas\c{c}ade inspection, as illustrated in Fig. \ref{fig:Coordination_illustration}, is to generate a proper maneuver pattern based on the fas\c{c}ade plane parameter $\bar{\nbf}_o$ and the inspection pattern, so as to effectively cover the surface of the inspection subject. 
	
	\begin{itemize}
		\item First, to guarantee proper camera observations as well as to avoid collisions, the UAV needs to keep a constant separation distance $d_s$ to the inspection plane, that is  
		\begin{equation}
		\label{sep_dis}
		\nbf_o^T\pbf_k + (d_o-d_s) = 0.  
		\end{equation}
		\item Second, to effectively cover the inspection surface, the UAV trajectory should follow some specific inspection pattern, such as the zig-zag path defined by the $d_c$ and $\nbf_c$, in Fig. \ref{fig:Coordination_illustration},
		\begin{equation}
		\label{pattern}
		\nbf_c^T\pbf_k - (n_ad_c+d_0)=0,
		\end{equation}
		where $\nbf_c$ is the inspection path incremental direction, $d_0$ is the initial displacement, and $d_c$ is the displacement between two consecutive inspection rounds. $n_a$ is the counting number of the inspection rounds and increases by one when a round reaches a boundary.
 
		\item Additionally, as effective inspection relies on stable camera motion, a reference velocity $v_{r}$ that provides a trade-off between efficiency and inspection quality is also required, that is  
			\begin{equation}
			\label{velocity}
			\nbf_p^T \vbf_k - (-1)^{n_a}v_{r} = 0,
			\end{equation}
	\end{itemize} 
	where $\nbf_p = \nbf_c\times \nbf_o$ is the direction of the reference velocity.
	For the inspection of typical buildings, we can set the inspection direction upward by letting $\nbf_c = \begin{bmatrix}
	0 & 0 &1
	\end{bmatrix}^T$. 
	
	With the definitions and models above, in this paper we consider to solve two consecutive problems, which are 1) plane parameters $\nbf_o,d_o$ estimation, and 2) following control of the planar object satisfying Equation (\ref{sep_dis}) to (\ref{velocity}), which will be described separately in Section \ref{Sec_III} and Section \ref{Sec_IV}. 
	
	\section{Plane Estimation}
	\label{Sec_III}
	In this part, the plane pose estimation problem is studied. Specifically, we consider that a camera is installed on the UAV to obtain 2D visual features of the fas\c{c}de plane, then the plane parameters in the camera frame  $\mathcal{C}$, denoted as $\nbf_o^c$,  $d_o^c$,  are estimated. Finally, based on the pose information of the UAV, the plane parameters in the global frame, $\nbf_o$,  $d_o$ can be estimated.
		\subsection{Plane observer design}
	To begin with, we simply consider one 3-D point in the camera frame $\mathcal{C}$, denoted as $\Xbf^c = [X^c,Y^c,Z^c]^T\in \mathbb{R}^3$, projected onto the 2-D image plane as 
	 $\sbf = \left[x, y \right]^T\in\mathbb{R}^2$ with the pin-hole model, then the following relationship holds,
	\begin{equation}
	\label{point_plane}
	\frac{1}{Z^c}  = -\frac{(\nbf^c_o)^T}{d^c_o}\begin{bmatrix}
	x\\y\\1
	\end{bmatrix},
	\end{equation}
	where $\nbf^c_o$ and $d^c_o$ define one plane that the 3-D point belongs to in the camera frame, and $Z^c$ is the depth of the 3-D point in the camera frame.
	
	The relationship between camera motion and pixel dynamics can be expressed with the well-known differential equation (\ref{visual_servoing}), 
	\begin{equation}
	\label{visual_servoing}
	\dot{\sbf} = \Lbf\begin{bmatrix}
	\vbf^c\\\pmb\omega^c
	\end{bmatrix},
	\end{equation}
	where {\small$\Lbf = \begin{bmatrix} 
	{-1 / Z^c} & {0} & {x / Z^c} & {x y} & {-1-x^{2}} & {y} \\
	{0} & {-1 / Z^c} & {y / Z^c} & {1+y^{2}} & {-x y} & {-x}
	\end{bmatrix}$}. 
	$\vbf^c = \begin{bmatrix}
	v_x & v_y & v_z
	\end{bmatrix}^T\in\mathbb{R}^3$ and $\pmb\omega^c = \begin{bmatrix}
	\omega_x & \omega_y & \omega_z
	\end{bmatrix}^T\in\mathbb{R}^3$ denote the translational and angular velocity of the camera frame respectively.
	
	For more intuitive formulation, let $ \pmb\chi = -\frac{\nbf_o}{d_o}\in\mathbb{R}^3$. Then by substituting equation (\ref{point_plane}) into equation (\ref{visual_servoing}), we have the feature motion expressed in terms of $\pmb\chi$ as
		\begin{equation}
		\begin{bmatrix}
		\dot{x}\\\dot{y}
		\end{bmatrix} = \begin{bmatrix}
		(xv_z -v_x)\pmb\chi^T\bar{\pbf} + {x y}\omega_x   -(1+x^{2})\omega_y + {y}\omega_z\\
		(yv_z - v_y)\pmb\chi^T\bar{\pbf} + (1+y^{2})\omega_x  {-x y}\omega_y  {-x}\omega_z
		\end{bmatrix}. 
		\end{equation}
{ 	\begin{remark}
		The above representation  $\pmb\chi\in\mathbb{R}^3$ is a minimum representation of a 3D plane similar to the closest point approach \cite{yang2019observability}. The implementation of inverse depth here instead of depth can simplify the modeling of the observer due to the camera model (\ref{point_plane}). The inverse depth is widely used in camera based applications, such as visual servoing \cite{chaumette2007visual}, SLAM \cite{civera2008inverse}.
	\end{remark}}
	The derivative of the plane parameter $\pmb\chi$ based on its definition is
	\begin{equation}
	\label{derivatives}
	\dot{\pmb\chi} = -\frac{\dot{\nbf}_od_o - \nbf_o\dot{d}_o}{d_o^2}.
	\end{equation}
	The derivatives of ${\nbf}_o$ and $d_o$, by the camera motion $\vbf^c$ and $\pmb\omega^c$, are
	\begin{align}
		\dot{\nbf}_o = \left[{\pmb\omega^c}\right]_{\times}\nbf_o,\\
		\dot{d}_o = \nbf_o \vbf^c.
	\end{align}
	Substitute the above equations into (\ref{derivatives}), the derivative of $\pmb\chi$ is 
	\begin{equation}
	\label{dotchi}
	\dot{\pmb\chi} = \pmb\chi\pmb\chi^T\vbf^c - \left[\pmb\omega^c\right]_{\times}\pmb\chi.
	\end{equation}
	Conclusively, the dynamics of the 2-D image feature and planar object parameter can be formulated as 
	\begin{equation}
	\left\{
	\begin{aligned}
	\dot{\sbf} &= \fbf_{\sbf}(\sbf,\ubf) + \pmb\Omega^T\pmb\chi,\\
	\dot{\pmb\chi} &= \fbf_{\pmb\chi}(\pmb\chi,\pbf,\ubf),
	\end{aligned}
	\right.
	\end{equation}
	where \[\fbf_{\sbf}(\sbf,\ubf) = \begin{bmatrix} 
	{x y} & {-1-x^{2}} & {y} \\
	{1+y^{2}} & {-x y} & {-x}
	\end{bmatrix}\pmb\omega,\]
	\[\pmb \Omega =\begin{bmatrix}
	x\\y\\1
	\end{bmatrix} \begin{bmatrix}
	xv_z - v_x & yv_z - v_y
	\end{bmatrix}, \]
	\[\fbf_{\pmb\chi}(\pmb\chi,\pbf,\ubf) =\pmb\chi\pmb\chi^T\vbf^c - \left[\pmb\omega^c\right]_{\times}\pmb\chi. \]
	Define the estimated tracking feature and plane parameter respectively as $\hat\sbf$ and $\hat{\pmb\chi}$, then the planar object parameter can be estimated with the following observer \cite{Spica2013a}:
	\begin{equation}
	\label{estimatro}
	\left\{
	\begin{aligned}
	\dot{\hat{\sbf}} =& \Lbf_{\pmb\omega}\pmb\omega + \pmb\Omega^T  \hat{\pmb\chi} + \Hbf \pmb\xi,\\
	\dot{\hat{\pmb\chi}} =& \hat{\pmb\chi}\hat{\pmb\chi}^T\vbf - \left[\pmb\omega\right]_{\times}\hat{\pmb\chi} + \pmb\lambda\pmb\Omega \pmb\xi,
	\end{aligned}\right.
	\end{equation}
	where $\pmb\xi= \sbf - \hat{\sbf}\in \mathbb{R}^2$ is the feature tracking error on the image plane. $\Hbf$ and $\pmb\lambda$ are the observer gains. 

		\begin{assumption}
			To simplify the problem, we assume that the motion of the camera can be accurately obtained with perfectly known movement of the UAV and translational relationship between the UAV frame and the camera frame.
		\end{assumption}
		With the estimated $\hat{\pmb\chi}$, the planar parameters in the camera frame can be calculated as
		\begin{align}
		\label{vectoranddepth}
		\hat\nbf_o^c = \frac{\hat{\pmb\chi}}{\|\hat{\pmb\chi}\|}, \hat d_o^c = \frac{1}{\|\hat{\pmb\chi}\|}.
		\end{align} 
		The parameters are transferred to the global frame as
		\begin{equation}\label{cameratoglobal}
			\begin{bmatrix}
			\nbf_o^g\\d_o^g
			\end{bmatrix} = \Mbf^{-T}\begin{bmatrix}
			\nbf^c_o\\d^c_o
			\end{bmatrix},
		\end{equation}
	 where $\nbf_o^g$ and $d_o^g$ denote the normal vector and depth in global frame. The superscript $g$ is omitted in the sequel without causing confusion.
{ 	 $\Mbf\in\mathtt{SE}(3)$ is a matrix of Lie group which is determined by the camera rotation and translation in the global frame\cite{hartley2003multiple}.}
		
	\subsection{Convergence Analysis}	
	Define the error vector of $\sbf$ and $\pmb\chi$ as $\ebf = \left[\pmb
	\xi^T, \pmb\varepsilon^T\right]^T$, where $\pmb\varepsilon = \hat{\pmb\chi} - \pmb\chi$, the error dynamics based on the above observer is 
	\begin{equation}
	\label{err_dyn}
	\left\{
	\begin{aligned}
	\dot{\pmb\xi} &= -\Hbf\pmb\xi + \pmb\Omega\pmb\varepsilon,\\
	\dot{\pmb\varepsilon} &= -\pmb\lambda\pmb\Omega\pmb\xi + \gbf(\pmb\varepsilon, \vbf,\pmb\omega),
	\end{aligned}
	\right.
	\end{equation}
	where $\gbf(\cdot)\triangleq \fbf_{\pmb\chi}(\pmb\chi,\pbf,\ubf) - \fbf_{\pmb\chi}(\hat{\pmb\chi},\pbf,\ubf)$. According to (\ref{dotchi}), the term $\gbf(\cdot)$ is a vanishing term w.r.t. the error $\pmb\varepsilon$, i.e., $\gbf(\pmb\varepsilon)\to \mathbf{0}$ as $\pmb\varepsilon\to \mathbf{0}$.
	\begin{lemma}(Persistency of Excitation\cite{Spica2013a})
		The origin of (\ref{err_dyn}) is locally exponentially stable if and only if the following persistency of excitation   condition holds:
		\begin{equation}
		\int_{t}^{t+T} \pmb\Omega(\tau) \pmb\Omega^{T}(\tau) d \tau \geq \beta \boldsymbol{I}_{n}>0, \quad \forall t \geq t_{0},
		\end{equation}
		for some $T>0$ and $\beta >0$, with $\Ibf_n$ being the $n\times n$ identity matrix.
	\end{lemma}
	Given one feature with position $[x,y]^T$ under camera velocity $[v_x,v_y,v_z]^T$, the square matrix $\pmb\Omega(t)\pmb\Omega^T(t)$ is 
	\begin{equation}
	\pmb\Omega(t)\pmb\Omega^T(t)=\left((xv_z - v_x)^2 + (yv_z -v_y)^2\right) \begin{bmatrix}
	x^2&xy&x\\
	xy & y^2& y\\
	x  & y  & 1
	\end{bmatrix}.
	\end{equation}
	It's obvious that, for one feature situation, $\text{rank}(\pmb\Omega(t)\pmb\Omega^T(t))\le 1$ and  $\pmb\chi\in \mathbb{R}^3$ cannot be estimated. By implementing multiple features, we have the following conclusion.	
	\begin{lemma}[Convergence conditions]
		\label{con_con}
		Consider $m$ features in the observer, the PE condition can be satisfied with the following conditions:
		\begin{itemize}
			\item There exists a minimum set of $4$ feature points satisfying that any three are not collinear,
			\item and the control velocity is not zero.
		\end{itemize}
	\end{lemma}
	\begin{proof}
		Define $\bar{\pmb\Omega} = \left[\pmb\Omega_1, \cdots, \pmb\Omega_m\right]$ as the stacked matrix of $m$ features. Then the squared matrix for $m$ features are 
		\begin{equation}
		\label{square_matrix}
		\bar{\pmb\Omega}\bar{\pmb\Omega}^T=	\sum_{i=1}^{m}\pmb\Omega_i\pmb\Omega_i^T = \begin{bmatrix}
		\alpha_1 \bar{\sbf}_1 &\alpha_2 \bar{\sbf}_2&\cdots&\alpha_m \bar{\sbf}_m
		\end{bmatrix}\begin{bmatrix}
		\alpha_1 \bar{\sbf}_1^T \\\alpha_2 \bar{\sbf}_2^T\\\vdots\\\alpha_m \bar{\sbf}_m^T
		\end{bmatrix},
		\end{equation} where $\alpha_i = \sqrt{(x_iv_z-v_x)^2 + (y_iv_z-v_y)}$, and $\bar{\sbf}_i = [x_i,y_i,1]^T = \frac{1}{Z_i}\Xbf_i$.
		
		We consider a more strict PE condition by enforcing $\bar{\pmb\Omega}\bar{\pmb\Omega}^T\ge \beta\Ibf_3$, or equivalently 
		\begin{equation}
		\label{mat_rank}
		\text{rank}(\begin{bmatrix}
		\alpha_i\frac{1}{Z_1}\Xbf_1 &\alpha_2 \frac{1}{Z_2}\Xbf_2&\cdots&\alpha_m \frac{1}{Z_m}\Xbf_m
		\end{bmatrix})=3.
		\end{equation} To guarantee that, we need at least three independent column  vectors $\Xbf_i$ in the matrix (\ref{mat_rank}), or equally, there are at least 3 points on the image plane which are neither coincident nor collinear, and, the corresponding $\alpha_i$ is not zero.
		Now we investigate the situations that make $\alpha_i= 0$.
		\begin{itemize}
			\item $v_x =v_y=v_z = 0$, which means the camera is stationary;
			\item $x_i = v_x/v_z, y_i = v_y/v_z, v_z \ne 0$, which means the camera is moving on the line connecting the feature $X_i$ and the camera center.
		\end{itemize}
		For the first case, all the elements $\alpha_i\frac{1}{Z_i}\Xbf_i$ are zeros, and therefore the rank of (\ref{mat_rank}) is zero. For the second case, assume there is a feature $i$ that satisfies $\alpha_i = 0$, then for all other $m-1$ features $j$ that do not coincide with $i$, $\alpha_j\ne 0$. 
		
		Finally, we can conclude that with the sufficient conditions proposed, the $rank(\bar{\pmb\Omega}\bar{\pmb\Omega}^T)\equiv 3$ with an arbitrary non-zero camera movement.
			\end{proof}

		As stated in \cite{Spica2014Active}, the convergence rate of the error system (\ref{err_dyn}) is related to the minimum eigenvalue of the matrix $\pmb\Omega\pmb\Omega^T$, which is decided by the velocity $\vbf$ and features $\sbf$.  
	\begin{lemma}[Convergence rate]
		\label{con_rate}
		Consider a scenario where the PE condition is satisfied, and given a reference direction of the moving camera, the convergence rate or the minimum eigenvalue of the matrix, denoted by $\lambda_{min}$, is 
		\begin{itemize}
			\item proportional to the square norm of velocity, $\lambda_{min}\propto \|\vbf\|^2$,
			\item a monotonic function with regard to the features set, namely, given two feature set $A\subseteq B$, $\lambda_{min}(A)\le \lambda_{min}(B)$.
		\end{itemize}.  
	\end{lemma}	
\begin{proof}
	The first property is straightforward based on the definition of the square matrix (\ref{square_matrix}).
	Denote the square matrix $\bar{\pmb\Omega}\bar{\pmb\Omega}^T$ corresponding to two different non-empty feature sets $A$ and $B$ as  $\pmb\Xi_{A}$ and $\pmb\Xi_{B}$ respectively. Based on (\ref{square_matrix}), we have $\pmb\Xi_{A}$, $\pmb\Xi_{B}$ and $\pmb\Xi_{B\backslash A}$ are all symmetric matrices. According to the Rayleigh-Ritz method \cite{yserentant2013short}, the minimal eigenvalue of a symmetric matrix $M$ satisfies
		\begin{equation}
		\lambda_{\min}(M) = \min_{\xbf\ne 0}\frac{\xbf^T M\xbf^T}{\xbf^T\xbf^T},
		\end{equation} where $\xbf$ is a vector of the same dimension.
	
As $\pmb\Xi_{B} = \pmb\Xi_{A} + \pmb\Xi_{B\backslash A}$, the minimum eigenvalue of matrix $\pmb\Xi_B$ is
	\begin{equation}
	\begin{aligned}
		\lambda_{\min}(\pmb\Xi_B) &=\lambda_{\min}(\pmb\Xi_B + \pmb\Xi_{B\backslash A})\\ &=\min\left(\frac{\xbf^T\pmb\Xi_A\xbf^T}{\xbf^T\xbf^T} + \frac{\xbf^T\pmb\Xi_{B\backslash A}\xbf^T}{\xbf^T\xbf^T}\right)\\
		&\ge  \min_{\xbf\ne 0}\frac{\xbf^T\pmb\Xi_A\xbf^T}{\xbf^T\xbf^T} + \min_{\xbf\ne 0}\frac{\xbf^T\pmb\Xi_{B\backslash A}\xbf^T}{\xbf^T\xbf^T}\\
		&=\lambda_{\min}(\pmb\Xi_A) + \lambda_{\min}(\pmb\Xi_{B\backslash A}).
		\end{aligned}
	\end{equation}
	Furthermore, according to the definition of the square matrix, the matrix $\pmb\Xi_{A}$ and $\pmb\Xi_{B}$ and $\pmb\Xi_{B\backslash A}$ are all positive semidefinite, which means $\lambda_{\min}(\pmb\Xi_{B\backslash A})\ge 0$, then we can have the conclusion $\lambda_{\min}(\pmb\Xi_B)\ge \lambda_{\min}(\pmb\Xi_A)$.
\end{proof}
Based on Lemma \ref{con_con} and \ref{con_rate}, we can conclude that, in spite that the convergence is guaranteed under very mild condition of the number of features, the convergence rate is positively correlated to the feature richness. To guarantee the stability and smoothness of the subsequent plane following control, we should adopt as many features as possible in the observer.
		\subsection{{ Connected planes estimation}}
	In a practical scenario with off-the-shelf camera sensors, the derivatives of the feature points are approximated with feature displacement between consecutive frames. Due to the limited field of view of a typical camera, the feature set $\sbf$ in the observer is dynamically updated with the motion of camera. 
	To realize continuous estimation on consecutive planes, the above observer, denoted as $PE\_Observer$ is implemented in Algorithm \ref{algorithm1}.
	\begin{algorithm} 
	\label{algorithm1}
	\SetAlgoLined
	\SetKwInput{kwInit}{Init} 
	\kwInit {Observer gains $\alpha, \beta$\; Plane estimate $\hat{\pmb\chi}_{0}$, feature set $\hat\sbf_0 = Null$\;  }
	\While{$k\ge0$}{
		Obtain one image from camera and extract one set of features $\mathcal{S}_k$\;
		Obtain the camera velocity and rotation velocity $\vbf_k,\pmb\omega_k$\;
		{
			Run feature matching between $\mathcal{S}_k$ and $\mathcal{S}_{k-1}$\;
			\For{ $\sbf_{j}\in\mathcal{S}_k$}{\uIf{$\sbf_{j}$ matches $\sbf_{l}\in\mathcal{S}_{k-1}$ }{$\hat\sbf_j = \hat{\sbf}_l$\;}\Else{$\hat\sbf_j = {\sbf}_j$\;}}
			}
		$\hat{\pmb\chi}_{k},\hat{\mathcal{S}}_{k} \leftarrow PE\_Observer(\hat{\pmb\chi}_k,\hat{\mathcal{S}_k},\vbf_k,\omega_k,\alpha,\beta)$\;

	}
	
	\caption{Plane parameters estimation}
\end{algorithm}
		%
		%
		%


	\section{Plane Following}
	\label{Sec_IV}
	\subsection{MPC based plane following}
	For the plane following problem described in Section \ref{objective}, we can formulate an objective function as
	\begin{equation}
	\begin{aligned}
		\min\limits_{\ubf} \quad \sum_{t=1}^{H} &\left(c_1\| \nbf_o^T \pbf_{k+t} - d_{ref}\|^2 +c_2\| \nbf_c^T\pbf_{k+t} - z_{ref}\|^2\right. \\ & \left. +c_3\| \nbf_{p}^T\vbf_{k+t} - v_{ref}\|^2  \right) +\|\ubf_{k+t-1}\|_R^2,
	\end{aligned}
	\end{equation}
	where $d_{ref} = d_c -d_o, z_{ref} = n_ad_c + d_0$, and $v_{ref} = (-1)^{n_a} v_{r}$.
	\begin{remark}
		The parameter $n_a$ is to count the inspection round of the whole mission process. When a boundary or a potential collision is detected, it automatically adds one, which results in that the reference $z_{ref}$ adds by a constant displacement $d_c$, and reverses the direction of the velocity $\vDash_{ref}$.	Although the proposed method counts on $n_a$, we only consider the estimation and path planning of the inspection area and the detection of boundaries is out of the scope of this paper.
	\end{remark}

	By transforming the position of the platform into a homogeneous coordinate, i.e. $\bar{\xbf} = \begin{bmatrix}
	x & y& z&1&v_x&v_y&v_z
	\end{bmatrix}^T$, the objective function can be further rewritten as a constant reference tracking problem as
	\begin{equation}
	\label{compact_obj}
	J =\sum_{t=1}^{H}	\|\Cbf\bar\xbf_{k+t} - \cbf_r\|_Q^2 + \|\ubf_{k+t-1}\|_R^2,
	\end{equation}
	where $\Cbf\in \mathbb{R}^{3\times7}$ is a matrix based on the plane vector, and $\cbf_r$ is a constant reference vector, respectively as 
	\begin{equation}
	\label{Co_def}
		\Cbf = \begin{bmatrix}
		\nbf_o^T & d_o & \mathbf{0}_{1\times 3}\\ \nbf_c^T & 0&\mathbf{0}_{1\times 3}\\
		\mathbf{0}_{1\times 3} &0& { {\nbf}_c\times \nbf_o}^T
		\end{bmatrix}, \quad  \cbf_r = \begin{bmatrix}
		d_{ref} \\ z_{ref}\\ v_{ref} 
		\end{bmatrix}.
	\end{equation}
	The weight matrix $Q = \text{diag}(\begin{bmatrix}
	c_1 & c_2 & c_3
	\end{bmatrix})$.
\begin{problem}\label{problem}By incorporating the dynamics of the system, the problem can be formulated as a typical MPC based reference tracking problem,
	\begin{subequations}
		\begin{eqnarray}
		\min_{\ubf_{k:k+H}} &J(\xbf_{k:k+H}, \ubf_{k:k+H-1} )\\
		\mathrm{s.t.} 
		& \xbf_{k+t} = \Abf \xbf_{k+t-1} + \Bbf \ubf_{k+t-1},\\
			&\label{vcon}\|\vbf_{k+t}\|\le \bar{v},\\
			&\label{ucon}\|\ubf_{k+t}\|\le\bar{u}, \\			
			&\label{terminal}\Cbf_{k}\bar\xbf_{k+H} - \cbf_r\in \mathcal{E}.
		\end{eqnarray}
	\end{subequations}

\end{problem}
\noindent The last term (\ref{terminal}) is the terminal condition to guarantee the close loop tracking performance, and $\mathcal{E}$ is the corresponding terminal set. The controller above is denoted as $PF\_MPC$ in the following parts.

\begin{remark}
{ 	Although not required by the plane observer, we consider to control the camera optical axis to be roughly aligned with the normal vector of the estimated plane to obtain better inspection details. For a camera rigidly connected to a UAV, this alignment can be achieved at lest on the $x-y$ plane by sending a yaw command according to plane estimation to the autopilot of the UAV. For a gimbaled camera, the alignment can be realized in 3D by controlling the gimbal angles.}

\end{remark}

		\subsection{Recursive Feasibility} 
		\label{feas_section}
		\begin{assumption}
			\label{feasibility}
			There exists a finite receding horizon $H$ satisfying that, for any $\xbf_0$ and $\hat{\Cbf}_0$ that belong to the initial sets $\hat\Cbf_{0}\in\mathcal{C}_0$ and $\xbf_0\in\mathcal{X}_0$ respectively, the solution for the MPC Problem \ref{problem} is feasible. 
		\end{assumption} 
		
		With the feasibility Assumption \ref{feasibility}, we denote the optimal control sequence at time $k$ as $\Ubf^*_{k} = \begin{bmatrix}
		(\ubf^*_{k|k})^T & (\ubf^*_{k+1|k})^T & \cdots & (\ubf^*_{k+H-1|k})^T 
		\end{bmatrix}^T$, which satisfies the terminal condition  $\hat\Cbf_{0}\bar\xbf_{k+H}- \cbf_r \in \mathcal{E}$.
		Next, we need to find a feasible solution at time $k+1$. Based on the optimal control sequence $\Ubf^*_{k}$, we compose one control sequence as $\Ubf'_{k+1} = \begin{bmatrix}
		(\ubf^*_{k+1|k})^T & (\ubf^*_{k+2|k})^T & \cdots & (\ubf^*_{k+H-1|k})^T  & (\ubf'_{k+H|k+1})^T
		\end{bmatrix}^T$. Then if $\ubf'_{k+H|k+1}$ satisfies the terminal condition 
		\begin{equation}
		\label{ck1}
		\hat{\Cbf}_{k+1}(\Abf\bar\xbf_{k+H|k} + \Bbf\ubf_{k+H|k+1})- \cbf_r\in\mathcal{E},
		\end{equation} and constraints (\ref{vcon},\ref{ucon}), $\Ubf'_{k+1}$ is one feasible solution.
		Apparently, when the estimator is stable, namely for a large enough $k$, $\hat\Cbf_{k+1} = \hat\Cbf_{k} = \Cbf_k$, then based on the definition of $\Cbf$, $\hat\Cbf_{k+1}\Abf\bar\xbf_{k+H|k} = \hat\Cbf_{k}\bar\xbf_k - \cbf_r\in\mathcal{E}$, i.e. the terminal condition can always be satisfied by letting $\ubf'_{k+H|k} = \mathbf{0}$, which means the solution is always feasible when the estimator is stable.
		
		Otherwise, let $\hat\Cbf_{k+1} = \hat\Cbf_k + \Delta \Cbf$ and we can obtain a control input that guarantees (\ref{ck1}) as 
		\begin{equation}
		\label{control}
		\ubf = \left(\left(\Delta\Cbf + \hat\Cbf_k\right)\Bbf\right)^{-1}\Delta \Cbf\Abf\xbf,
		\end{equation}
		where 
		\begin{equation}\label{deltaC}\Delta \Cbf = \begin{bmatrix}
		\Delta\nbf_o & \Delta d_o & \mathbf{0}_{1\times 3}\\
		\mathbf{0}_{1\times3}& 0&\mathbf{0}_{1\times3}\\
		\mathbf{0}_{1\times3} &0 & \Delta \nbf_o\times \nbf_c
		\end{bmatrix}.\end{equation}
		The control input (\ref{control}) may violate the control input constraint for a large change $\Delta \Cbf$. In order to guarantee the control feasibility, an estimation update strategy for the sampling of $\hat{\pmb\chi}$ is deigned to limit $\Delta \Cbf$ as 
		\begin{equation}
		\label{chiupdate}
		\hat{\pmb\chi}_{s,k+1} = \hat{\pmb\chi}_{s,k} + \gamma(\hat{\pmb\chi}_{(k+1)\delta t_s} - \hat{\pmb\chi}_{s,k}),
		\end{equation}
		where $\delta t_s$ is the sampling time interval, $\gamma \in \left[0,1\right]$ is an update step factor and should be adaptively changed so that the control input (\ref{control}) is always feasible. { Denote $\Cbf$ corresponding to $\hat{\pmb\chi}_{s,k}$ as $\hat\Cbf_{s,k}$, then  	$\hat\Cbf_{s,k+1 } $ and $\Delta \Cbf_{s,k}$ can be treated as functions of $\gamma$
		\begin{equation}
			\hat\Cbf_{s,k+1}(\gamma) = \hat\Cbf_{s,k} + \Delta \Cbf_{s,k}(\gamma).
		\end{equation} 
		The largest feasible $\gamma$, denoted as $\bar\gamma$, can be calculated numerically by combining (\ref{chiupdate}) with the below equation :
		\begin{equation}
		\|\ubf\| = \|\left(\left(\Delta\Cbf_{s,k} + \hat\Cbf_{s,k}(\gamma)\right)\Bbf\right)^{-1}\Delta \Cbf_{s,k}(\gamma)\Abf\bar\xbf\| = \bar{u}.
		\end{equation} }\newline
	The adaptive step size is $\gamma = \min\{1,\bar\gamma\}$. {Obviously, when the difference between the current estimate from observer $\pmb\chi$ and the last sampled estimate $\pmb\chi_{s}$ is so large that the control input calculated directly from (\ref{control}) is not feasible, $\gamma<1$ serves as a discount factor, and when the tracking error becomes small enough, $\gamma = 1$.}
		\begin{lemma}
			Based on the proposed plane observer as well as the update rule (\ref{chiupdate}), the sampled $\hat{\pmb\chi}_s$ ultimately converges to the true planar parameters $\pmb\chi$, and during the updating process, the MPC is recursively feasible.
		\end{lemma}   
		\begin{proof}
			Based on the updating rule (\ref{chiupdate}), we have
			\begin{equation}
					\Delta\hat{\pmb\chi}_{s,k}	\triangleq\hat{\pmb\chi}_{s,k+1}- \hat{\pmb\chi}_{s,k} =  \gamma(\hat{\pmb\chi}_{(k+1)\delta t_s} - \hat{\pmb\chi}_{s,k}).
			\end{equation}
			For a large enough sampling time instance $k‘$, 
			the estimation from the proposed plane observer converges to the true value, $\hat{\pmb\chi}_{k'\delta t_s}\to \pmb\chi$, therefore for $k>k'$, the error between the true value and the sampled value evolves as 
						\begin{equation}
						 ({\hat{\pmb\chi}}_{s,k+1} - \pmb\chi) =  (1-\gamma)	( \hat{\pmb\chi}_{s,k}- {\pmb\chi}).
						\end{equation}
						According to the definition of $\gamma$, when $ \hat{\pmb\chi}_{s,k}- {\pmb\chi}$ is large, $0<\gamma<1$, and the true value tracking error monotonically decreases.  When the tracking error becomes small enough at time instance $k''>k'$ so that the feasibility of (\ref{control}) is guaranteed, $\gamma = 1$, and ${\hat{\pmb\chi}}_{s,k''+1} - \pmb\chi = 0$, namely the sampled $\hat{\pmb\chi}_{s,k}$ converges to the true planar parameters $\pmb\chi$. 
						
			Obviously, if there exists a feasible solution at time $k$, then according to the update rule in (\ref{chiupdate}), the terminal condition at time $k+1$ can always be satisfied given a proper control, as in (\ref{ck1}), which means that there is always a feasible solution at $k+1$. The recursive feasibility is proved.
		\end{proof}
		Apparently, the adaptive updating step size slows down the convergence rate of the planar objective tracking, to an acceptable level of the platform maneuverability.
	
		\subsection{Tracking stability} 
		To investigate the stability of the proposed MPC controller, we define the planar object tracking error as 
		\begin{equation}
		\label{track_error}
		\ebf = \Cbf \bar\xbf - \cbf_r,
		\end{equation}
		which corresponds to the first part of the objective function (\ref{compact_obj}).
		\begin{theorem}
			 Based on the proposed planar object estimator	(\ref{estimatro}) and its corresponding discrete time update rule (\ref{chiupdate}), and given that the initial state satisfies the feasible Assumption \ref{feasibility}, the proposed receding horizon controller can properly track the reference $\cbf_r$, i.e. the tracking error $\ebf$ defined above is asymptotically stable.
		\end{theorem}	
		\begin{proof}
			According to the feasibility Assumption (\ref{feasibility}), the solution of Problem (\ref{problem}) is feasible for a state $\xbf_k\in \mathcal{X}$. 
			Denote the optimal solution at time instance $k$ as $\Ubf^*_k = \begin{bmatrix}
			(\ubf^*_{k|k})^T & (\ubf^*_{k|k})^T&\cdots& (\ubf^*_{k+H-1|k})^T
			\end{bmatrix}^T$. The corresponding objective function is 
			\begin{equation}
			J^*(k) = \sum_{i=1}^{H}\|\hat{\Cbf}_k\bar\xbf^*_{k+i|k}- \cbf_r\|^2_Q + \|\ubf^*_{k+i-1|k}\|_R^2.
			\end{equation}
			Similar to Section \ref{feas_section}, we have one feasible solution at time $k+1$ as 
			$\begin{bmatrix}
			(\ubf'_{k+1|k+1})^T & (\ubf'_{k+2|k+1})^T&\cdots& (\ubf'_{k+H|k+1})^T	
			\end{bmatrix} = \begin{bmatrix}
			(\ubf^*_{k+1|k})^T & (\ubf^*_{k+2|k+1})^T&\cdots& (\ubf'_{k+H|k+1})^T\end{bmatrix}.$ Using the terminal constraint, we have $\hat\Cbf_k\bar\xbf^*_{k+H=|k}- \cbf_r\in \mathcal{E}$. Letting $\ubf'_{k+H|k+1} = \mathbf{0}$, we have $\hat\Cbf_k\bar\xbf^*_{k+H+1|k}= \cbf_r$. Thanks to the local convergence property of the estimator, for $k\to \infty$, $\hat\Cbf_k\to \Cbf$. Then we have the objective function based on the feasible solution as 
			\begin{small}
				\begin{equation}
				\begin{aligned}
				J'(k+1) = &\sum_{i=2}^{H+1}\|\hat{\Cbf}_{k+1}\xbf'_{k+i|k+1} - \cbf_r\|^2_Q +\|\ubf'_{k+i|k+1}\|_R^2\\
				= &\sum_{i=2}^{H}\|\hat{\Cbf}_{k+1}\xbf^*_{k+i|k} - \cbf_r\|^2_Q + \|\ubf^*_{k+i|k}\|_R^2.
				\end{aligned}
				\end{equation}
			\end{small}
			Hence, 
			\begin{equation}
			\begin{aligned}
			J^*(k+1) - J^*(k)\le J'(k+1) - J^*(k)\\
			=-\|\hat{\Cbf}_{k}\xbf_k-\cbf_r\|^2_Q - \|\ubf_{k|k}\|_R^2 \le 0.
			\end{aligned}
			\end{equation}
			As $\hat{C}_k\to {\Cbf}$ as $t\to 0$, we can conclude that the tracking error defined above is asymptotically stable.
		\end{proof}
		\subsection{Plan following under connected planes}
		Based on the above estimator and controller, an algorithm that enables concatenated planes following is described in Algorithm \ref{algorithm2}. 
	\begin{algorithm} 
		\label{algorithm2}
		\SetAlgoLined
		\SetKwInput{kwInit}{Init} 
		\kwInit {Inspection parameters: $v_{r}$, $d_s, n_a, d_c$\;
			\quad\quad MPC parameters:$Q, R, H$, initial state $\xbf_0,\Cbf_0$\;}
			\While{not stop inspection}{
				Subscribe plane estimation information from Algorithm \ref{algorithm1}\;
				Update $\hat{\pmb\chi}_{s,k}$ according to (\ref{chiupdate})\;
				Calculate $\hat\Cbf_{s,k}$ according to (\ref{vectoranddepth}),(\ref{cameratoglobal}),(\ref{Co_def})\;
				\If{Is arrived at boundary}{Update the reference $\cbf_r$  according to (\ref{pattern}),(\ref{velocity})\;$a \leftarrow a +1$}
				Calculate the MPC controller
				$\ubf_{k:k+H-1},\xbf_{k+1:k+H}\leftarrow   PF\_MPC(\xbf_k,\Cbf_k) $ \;
				Execute $\ubf_k$.
				
			}
			
			\caption{Plane Follower}
		\end{algorithm}

	\section{Simulations}
	This subsection presents numerical simulations based on MATLAB \cite{MATLAB_2015} and the Machine Vision Toolbox \cite{corke2007matlab}, to test the performance of the proposed  plane estimation and plane following methods. Specifically, the simulation scenario is described as a projective camera model observing synthetic features so as to estimate and follow the planes.
	\subsection{Plane Estimation}
	In the first simulation scenario, the performance of the plane estimator based on different settings, such as different translational velocities and different level of feature richness, are investigated. The synthetic 3-D feature points are randomly placed on the plane determined by the plane normal vector $\nbf_o = [-0.2425, -0.9701 , 0]^T$, and feature depth $d_o = 9.7011$. One camera with limited horizontal FOV $46^\circ$ and vertical FOV $38^\circ$, is initialized with the transformation matrix  
	\[\Mbf(0)=\begin{bmatrix}
		-1& 0& 0& 40\\
		0 & 0& -1 & 20\\
		0 & -1 & 0 & 5\\
		0 & 0& 0& 1
	\end{bmatrix}.\] 
	\begin{figure}
\centering
\includegraphics[width=0.8\linewidth]{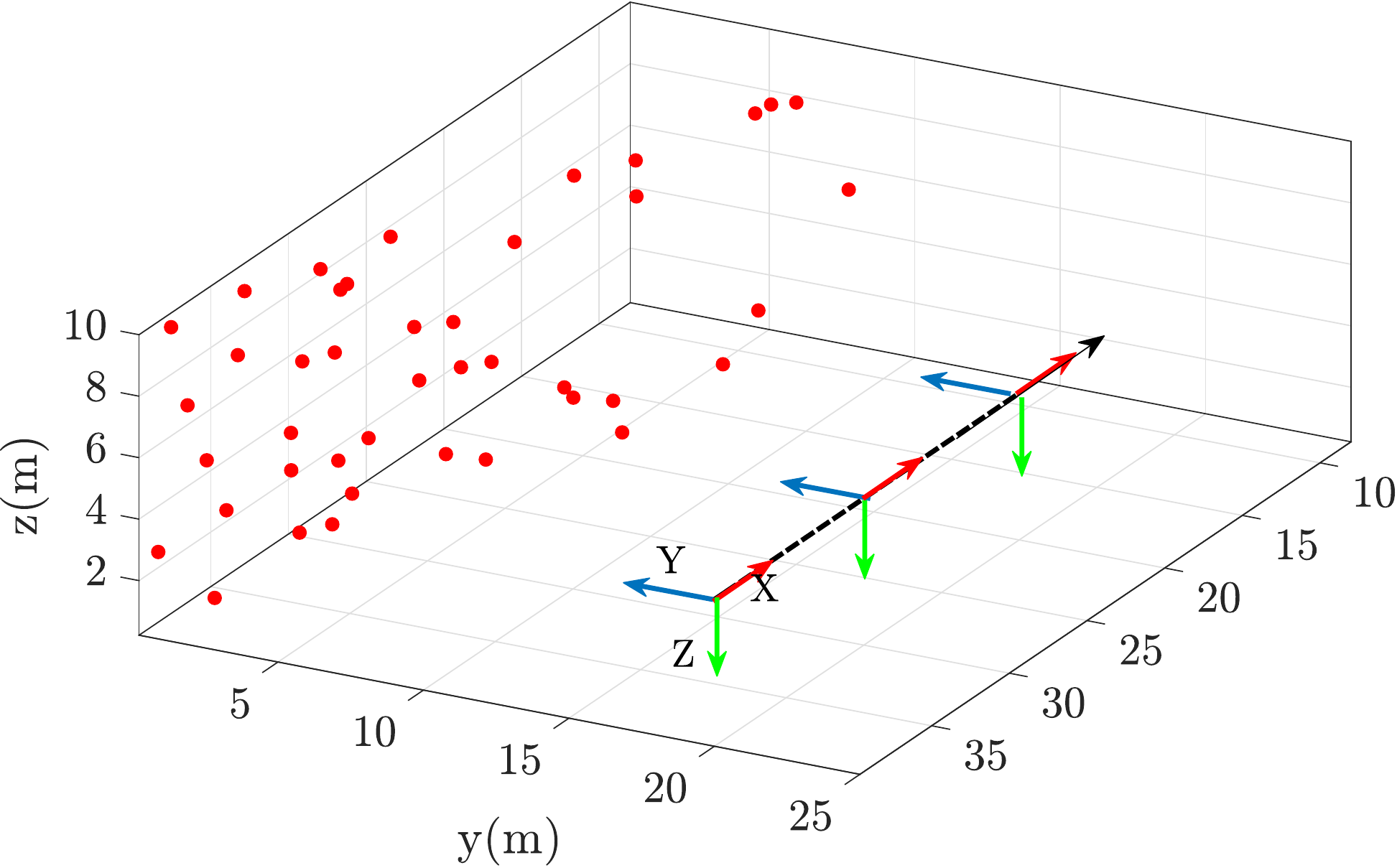}
\caption{The simulation scenario with synthetic features in red dot and camera pose in blue.}
\label{fig:illustrator-crop}
\end{figure}\newline
The observer gains are set as $H = 12\Ibf$, and $\pmb\lambda = 0.95\Ibf$.
The results of 40 seconds simulation based on the proposed observer running at 10 Hz is plotted in Figure \ref{fig:plane_norm_depth-crop}. In this simulation, 100 random features are simulated, and the camera is moving with a constant velocity $\vbf= [0.5;0;0]$. Fig.\ref{fig:plane_norm_depth-crop} shows the convergence behavior of the estimation for both plane normal vector and plane depth.

Moreover, the effect of the translational velocity and the number of features $n_s$ on the convergence rate of the estimation is investigated. The plane normal vector estimation error and the estimation error of depth $d_o$ are defined respectively as  \[e_{{\nbf}_o} = \cos^{-1}(\nbf_o^T\hat\nbf_o),\quad e_{d_o} = d_o- \hat d_o.\]  The evolutions of the error  under different translational velocity $v_x = 0.1$m/s, $0.25$m/s, $0.5$m/s are plotted in Figure \ref{fig:error_vel_all-crop}, which shows that faster movement results in a faster convergence rate.  Similarly, the evolution of the errors under different feature numbers are plotted in Figure \ref{fig:error_seed_all-crop}, which shows that the errors converge to zero faster as more features are implemented in the observer. The above two results are consistent with the conclusions of Lemma \ref{con_rate}.
	\begin{figure}
\centering
\includegraphics[width=0.85\linewidth]{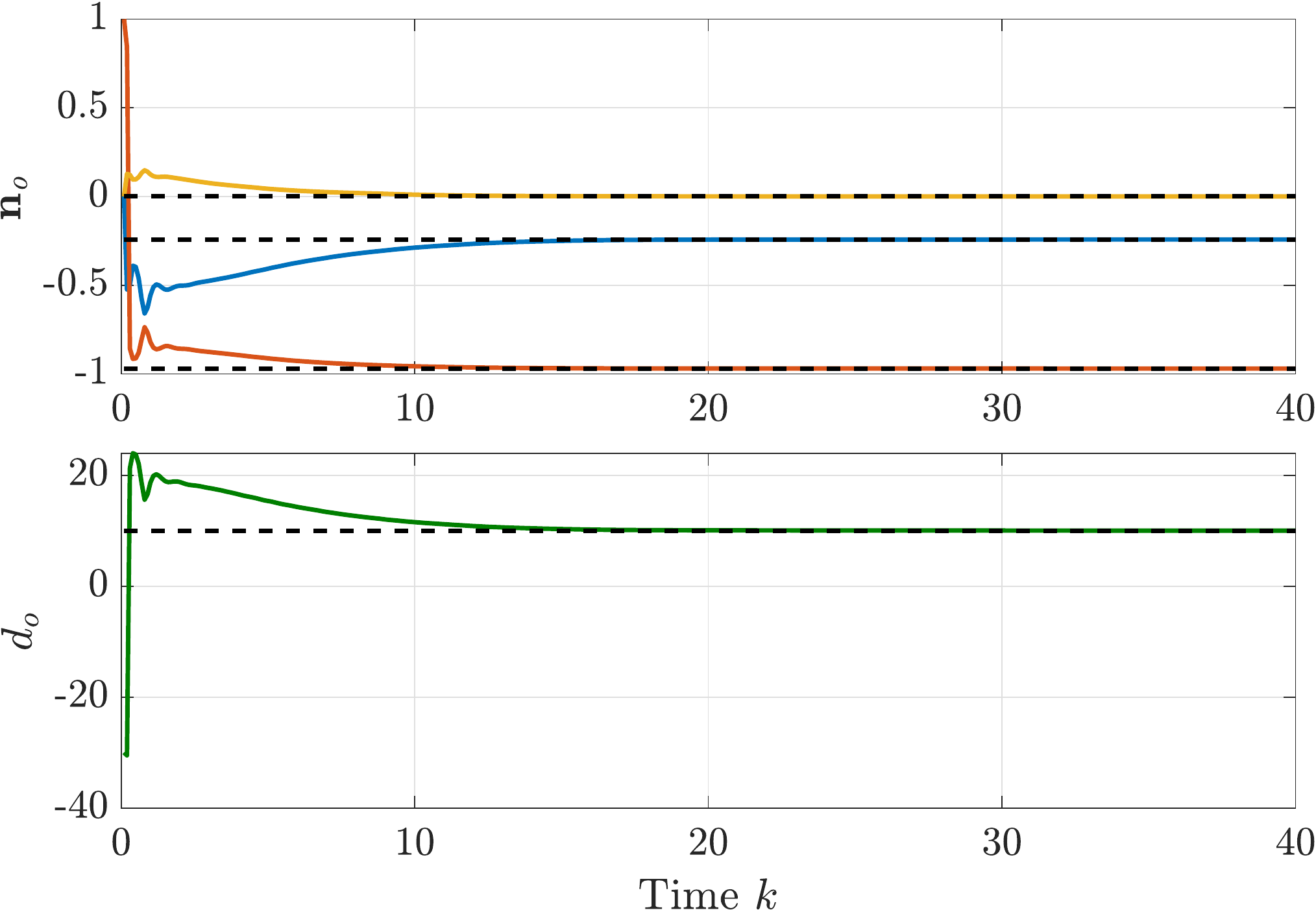}
\caption{The plane vectors $\nbf_o$, and $d_o$ in colors w.r.t to ground-truth in black dash-lines.}
\label{fig:plane_norm_depth-crop}
\end{figure}
\begin{figure}
\centering
\includegraphics[width=0.85\linewidth]{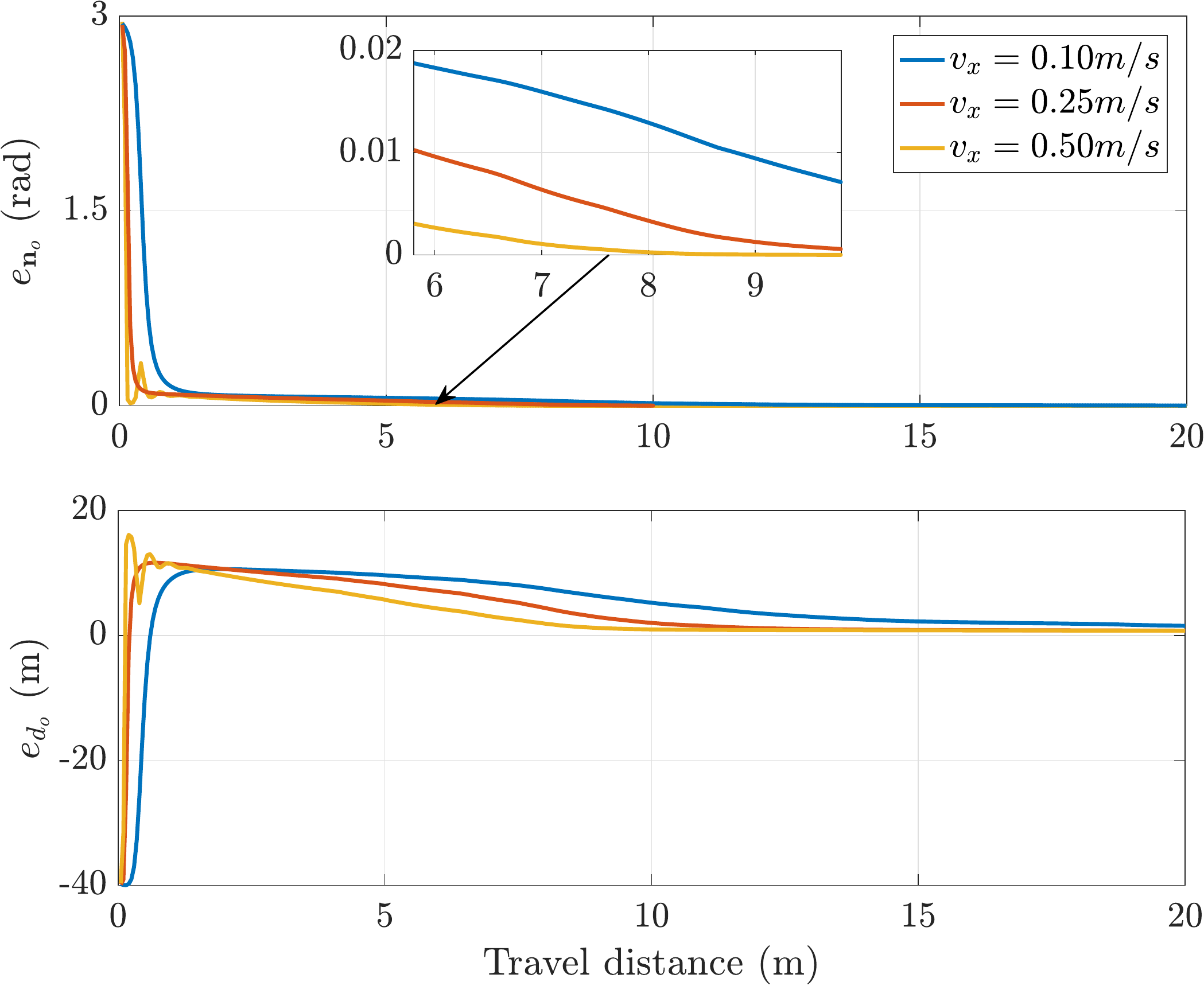}
\caption{The errors $e_{\nbf_o}$ and $e_{d_o}$ under different translation velocities $v_x = 0.1 \text{m/s},0.25 \text{m/s},0.5 \text{m/s}$ respectively.}
\label{fig:error_vel_all-crop}
\end{figure}
\begin{figure}
\centering
\includegraphics[width=0.85\linewidth]{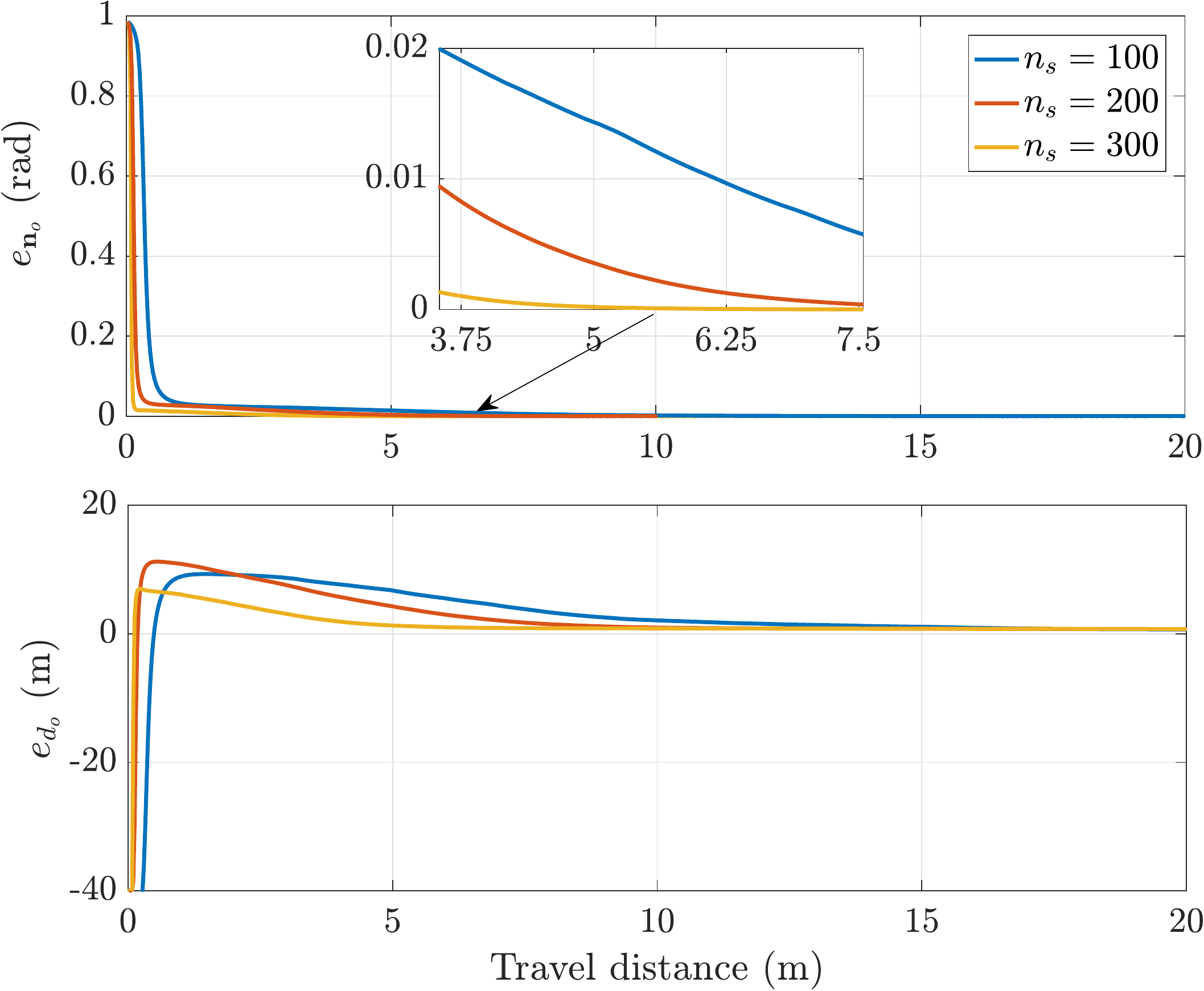}
\caption{The errors $e_{\nbf_o}$ and $e_{d_o}$ under different feature richness, with $n_s = 100, 200, 300$ respectively.}
\label{fig:error_seed_all-crop}
\end{figure}

{ Finally, a comparative study of the proposed method along with a typical EKF based estimator (EKF) and a homography-constrained method (Homo.)\cite{Giordano2014An} is provided under a similar scenario as above. Specifically, the covariance of the measurement noise of $\mathbf{s}$ is set as $\Sigma_\mathbf{s} = \text{diag}([0.001,0.001])$. The results of one trial using the three methods are plotted in Fig. \ref{fig:errorcompv3-crop}. Additionally, 100 Mento-Carlo trials are carried out, and the averaged estimation errors with different methods are plotted in Figure \ref{fig:error_combined-crop}.   
\begin{figure}
\centering
\includegraphics[width=0.85\linewidth]{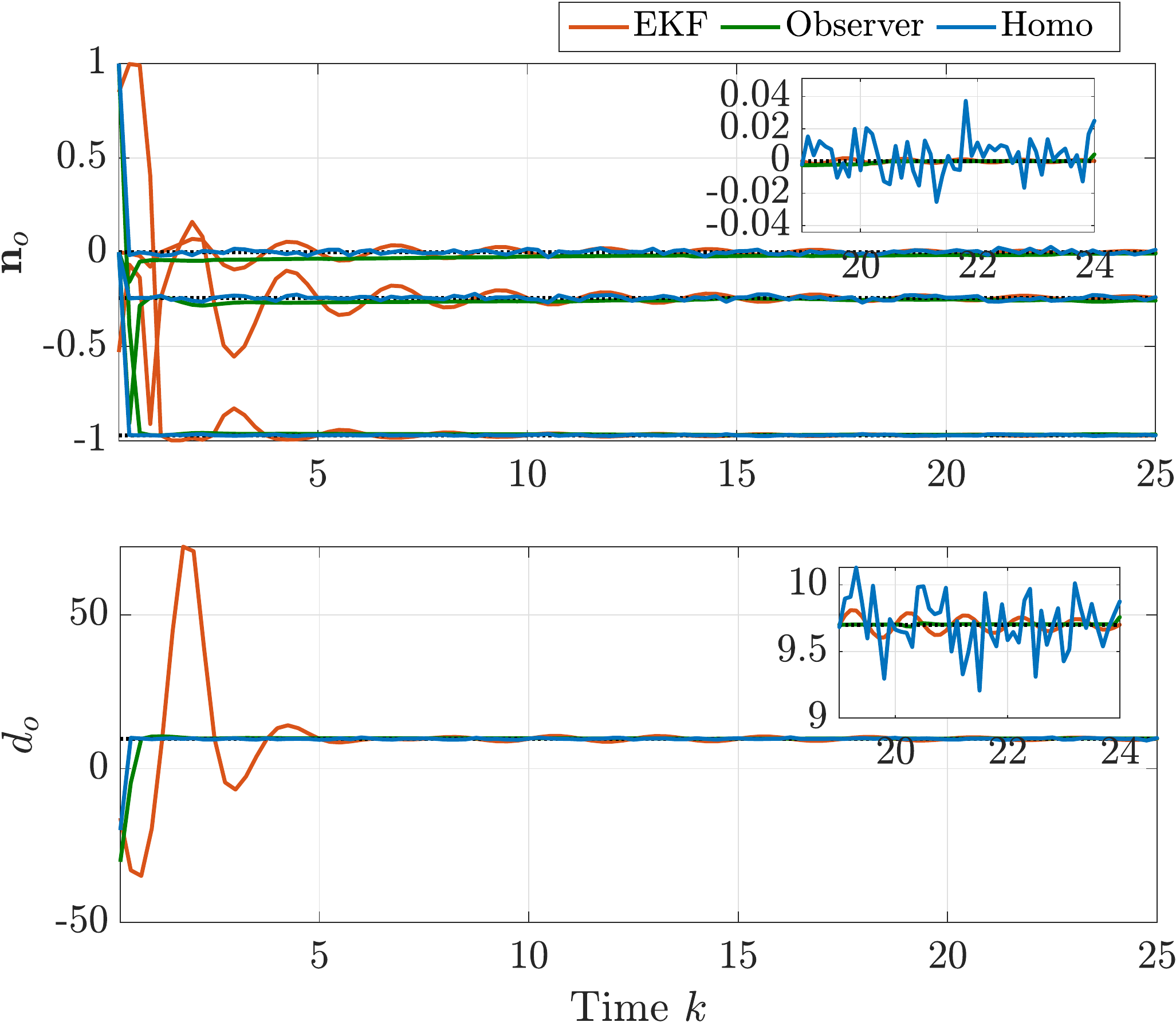}
\caption{{ The estimation result of the proposed result v.s. an EKF estimator and the homograph based method \cite{Giordano2014An}.}}
\label{fig:errorcompv3-crop}
\end{figure}
From Fig. \ref{fig:errorcompv3-crop} and Fig. \ref{fig:error_combined-crop}, it's apparent that the convergence rate of the proposed method is faster compared to the EKF, and the homography based method can achieve the fastest convergence rate. Nevertheless, the homography decomposition method only relies on image correspondences and is very
sensitive to measurement noise, as shown in the amplified sub-figures in Fig. \ref{fig:errorcompv3-crop}. The oscillation of estimation may produce unwanted plane following behavior.
On the contrary, the EKF and proposed estimator recursively estimate the depth of the plane, and therefore can achieve much smoother estimation and plane following results. 
\begin{figure}
\centering
\includegraphics[width=0.85\linewidth]{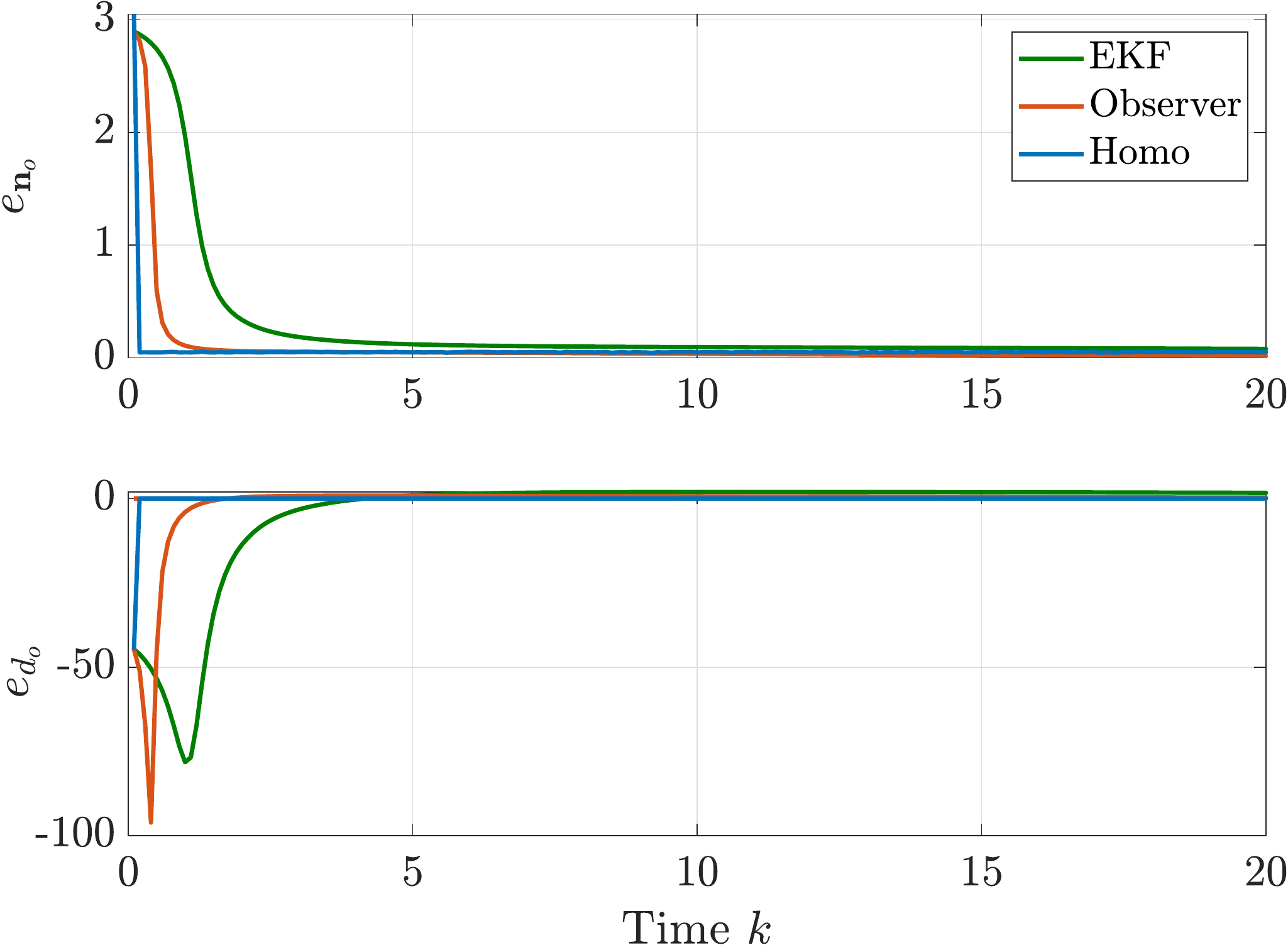}
\caption{{ The errors $e_{\nbf_o}$ and $e_{d_o}$ of the proposed method v.s. an EKF estimator and the homograph based method \cite{Giordano2014An} in 100 Mento-Carlo simulations.}}
\label{fig:error_combined-crop}
\end{figure}

}

\subsection{Plane Following}
The plane following based on Algorithm \ref{algorithm2} is simulated here. Specifically, we consider the continuous following of two connected planes, with the plane parameters denoted respectively as $\bar{\nbf}_{o1} = [-0.2425, -0.9701 , 0,9.7011]$, and   $\bar{\nbf}_{o2} = [-0.9701, -0.2425 , 0,9.7011]$.
The plane following parameters are set as $d_s = 10  $m, $v_{ref} = 1$m/s. The inspection pattern is set as $d_0 = 5, d_c = 2$m.
The UAV/Camera velocity and acceleration constraints are set respectively as $\|\vbf\|_{\infty} = 3$m/s and $\|\ubf\|_{\infty} = $0.5m/s$^2$.

The plane following trajectory is plotted in Fig. \ref{fig:trajectory_two_planes-crop} which indicates that proper path following/inspection behavior is achieved. The corresponding velocity and control input are plotted in Figure \ref{fig:velocity_control_two_planes-crop}.
The estimation results during the path following process is provided in Fig. \ref{fig:estimate_two_planes-crop}, which indicates that the proposed method can effectively estimate the plane parameters and quickly responds to the plane changes. In addition, the sampling strategy (\ref{chiupdate}), although leading to some latency, can effectively guarantee the feasibility of the proposed controller by providing smoother estimation results.  Finally, the performance, i.e. the path following error, defined in (\ref{track_error}), is plotted in Figure \ref{fig:Tracking_err_two_planes-crop}. Conclusively, the proposed estimation and control methods can effectively fulfill the inspection task.
\begin{figure}
\centering
\includegraphics[width=0.85\linewidth]{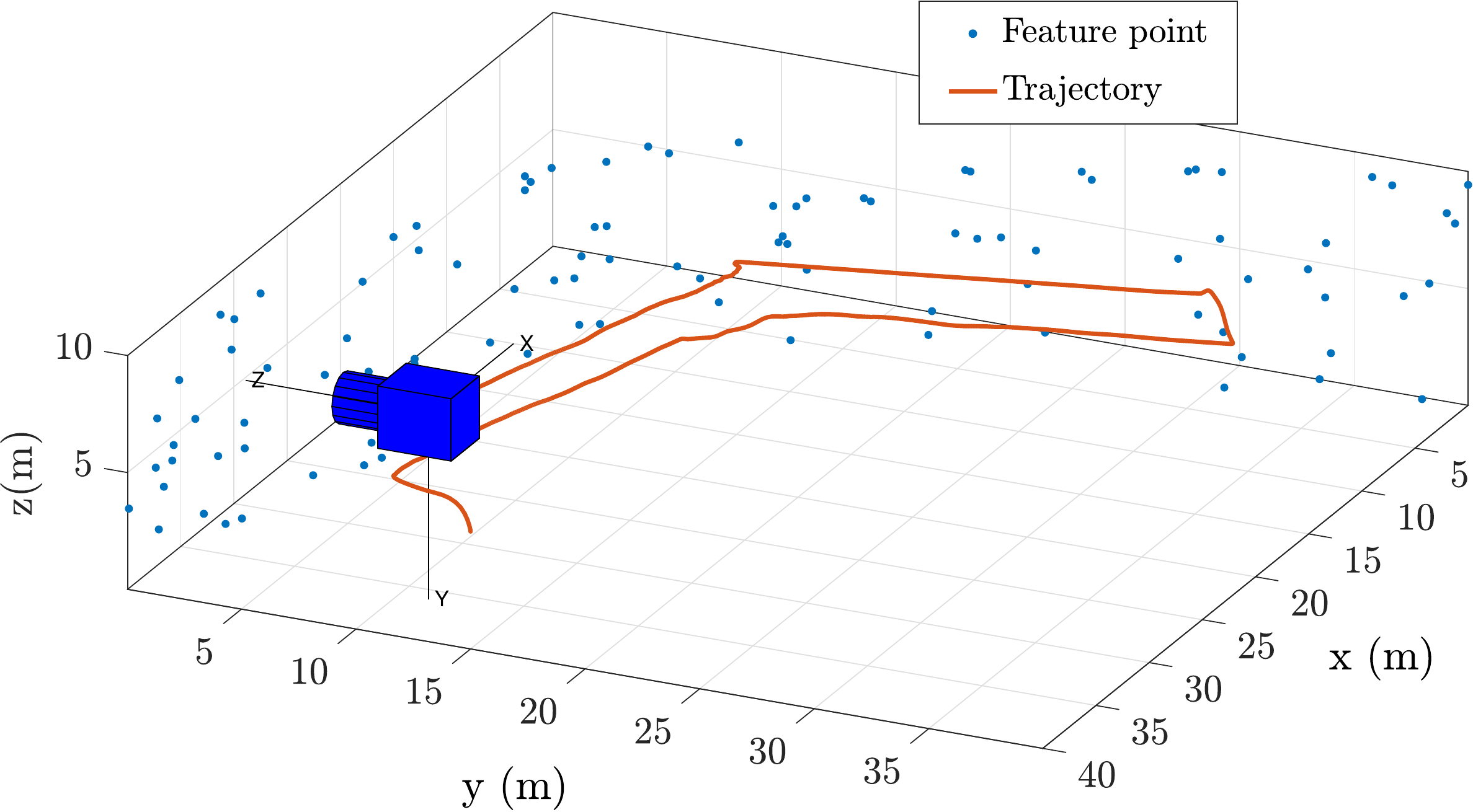}
\caption{The plane following trajectory w.r.t two connected planes.}
\label{fig:trajectory_two_planes-crop}
\end{figure}
\begin{figure}
	\centering
	\includegraphics[width=0.85\linewidth]{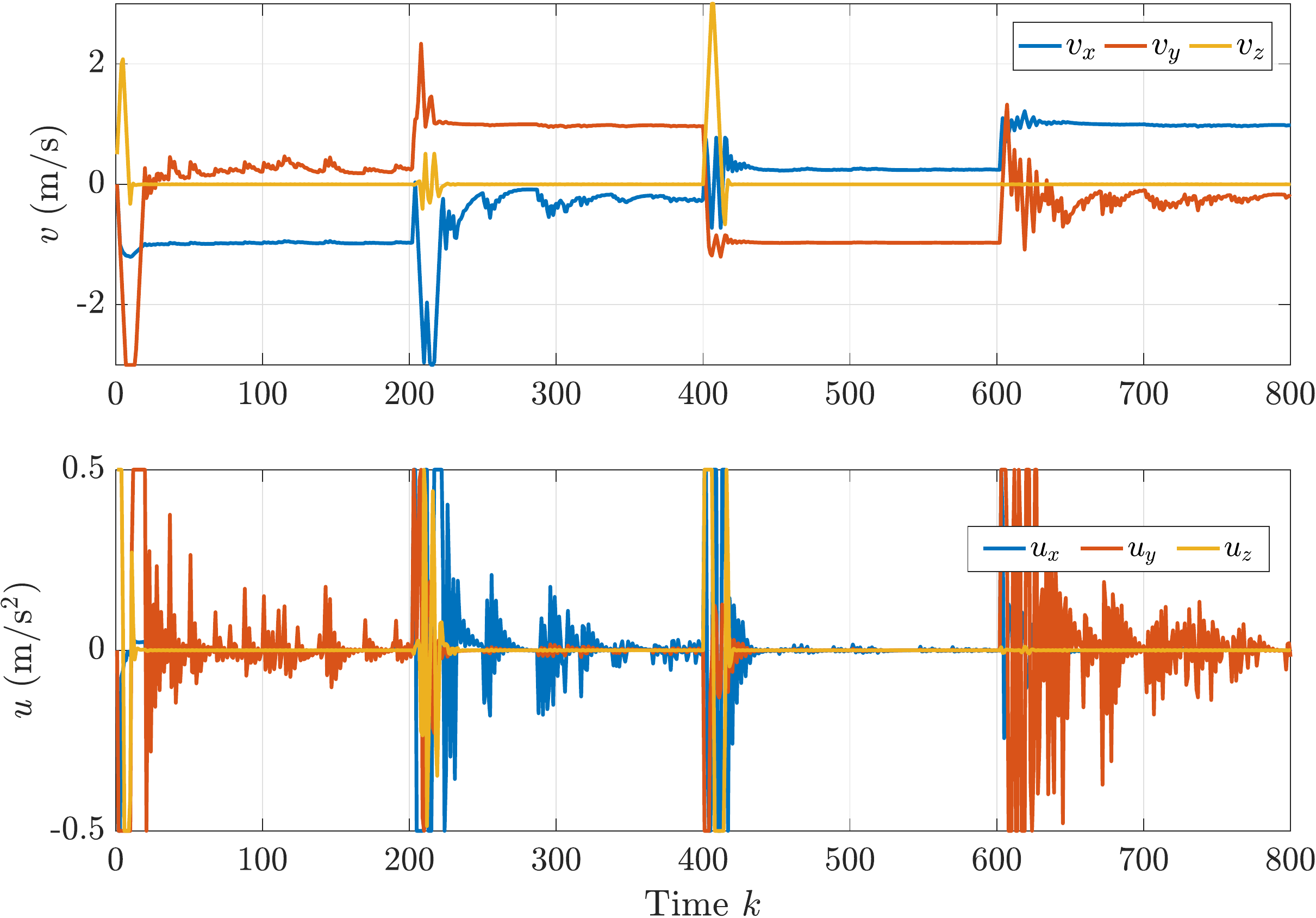}
	\caption{The velocity and control output of the plane following.}
	\label{fig:velocity_control_two_planes-crop}
\end{figure}
\begin{figure}
\centering
\includegraphics[width=0.85\linewidth]{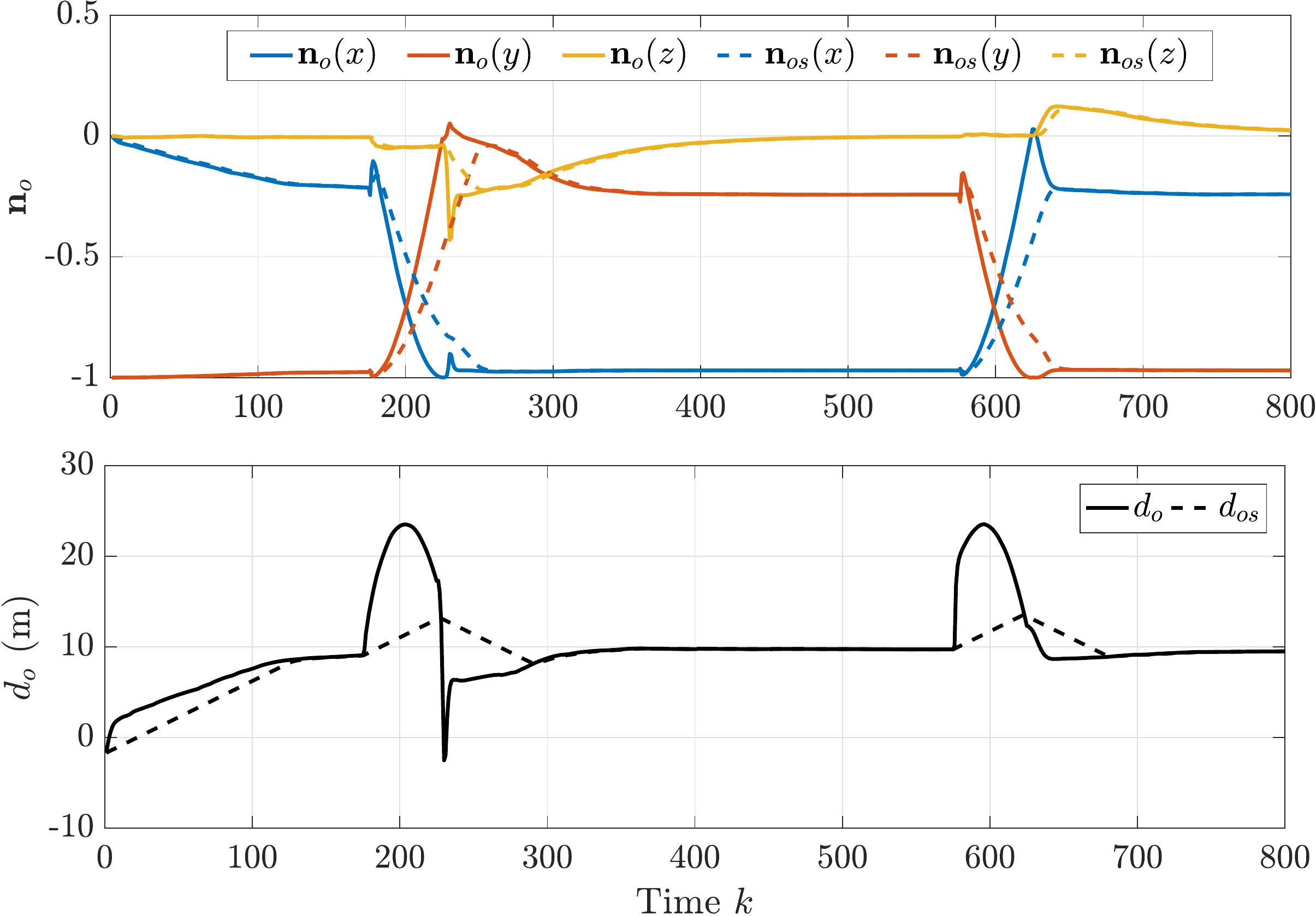}
\caption{The plane parameters estimation with the original estimation in solid lines and the sampled according to (\ref{chiupdate}) in dashed lines.}
\label{fig:estimate_two_planes-crop}
\end{figure}
\begin{figure}
\centering
\includegraphics[width=0.85\linewidth]{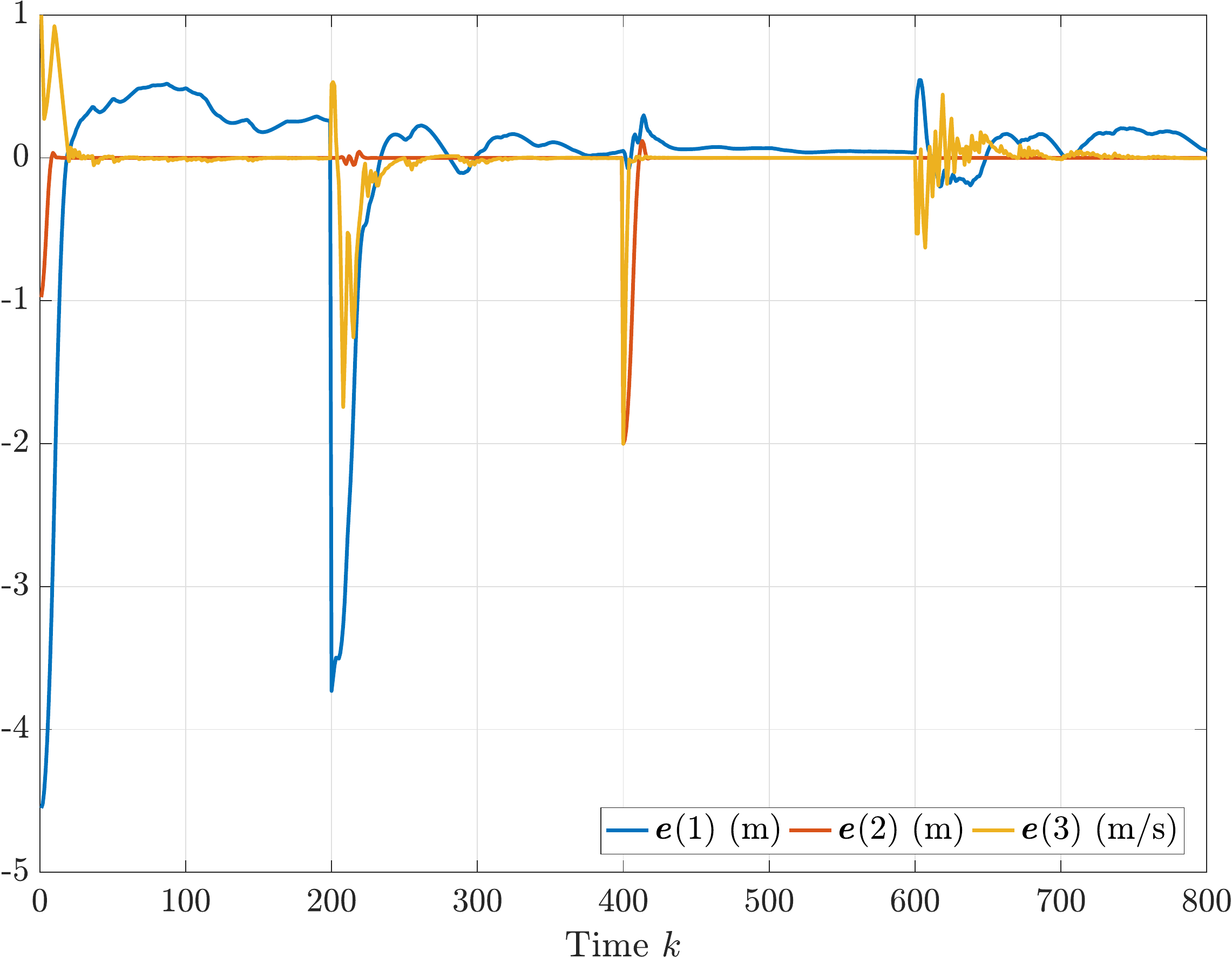}
\caption{The plane following error $\ebf$ defined in (\ref{track_error}).}
\label{fig:Tracking_err_two_planes-crop}
\end{figure}
\section{Experiements}
In this section, the proposed plane estimation and following methods are tested based on field experiment trials. 

A DJI M100 quadrotor carrying an ob-board camera as shown in Fig. \ref{fig:hardware} is implemented to inspect a building. For safety and comparison purpose, an on-board LiDAR is also installed to provide collision avoidance function. {Due to its high range measurement accuracy, around $0.015$ m $\sim0.1$ m, the LiDAR also serves as the ground truth for the path estimation and path following}. The software flow of the experiment system is given in Fig. \ref{fig:softwareflow}. Besides the inspection Algorithm \ref{algorithm1} and \ref{algorithm2}, the VINS-mono package \cite{qin2018vins} is  also implemented to provide the odometry information of the UAV. The interaction between the functional modules and the platform is through the DJI onboard SDK \cite{osdk}.

In the inspection task, the inspected building is shown as Figure \ref{fig:scene}, The separation distance is $d_s=7.5$m, and the inspection pattern is specified with $d_0 = 4$m, and $d_c = 2$m. For safety purpose, horizontal operation boundaries are set as $p_y\in [0,10]$.

{During the inspection process, the sampled estimation of the building plane  $\hat{\nbf}_{os}$ and $\hat{d}_{os}$ is shown in Fig. \ref{fig:est_exp-crop}. For comparison purpose, the actual separation distance of the UAV trajectory to the building measured from the on-board LiDAR is also plotted. 
The corresponding estimation errors, expressed with $e_{\nbf_o}$ and $e_{d_o}$, are plotted in Fig. \ref{fig:est_err_exp-crop}, which shows that the errors of the estimation are within an acceptable level, $e_{\nbf_o}< 0.2$ rad, and  $e_{d_o}< 0.2$ m. 


The inspection trajectory of the quadrotor is shown in Fig. \ref{fig:scene} with generated point cloud\footnote{The point cloud is generated offline using recorded data from the LiDAR and the  A-LOAM mapping package \url{https://github.com/HKUST-Aerial-Robotics/A-LOAM}.}, which demonstrates that a proper inspection behavior is achieved. The corresponding velocity and control input are plotted in Fig. \ref{fig:vel_exp-crop} and \ref{fig:acc}, respectively.
To further investigate the plane following performance, the plane path following error defined in (\ref{track_error}) is plotted in Fig. \ref{fig:tracking_err_exp-crop}. As indicated in Fig. \ref{fig:tracking_err_exp-crop}, the tracking errors of the separation distance, displacement, and reference velocity are all within acceptable bounds. Note that the plane estimation and path following also depend on the odometry information provided by the VINS-mono, as shown in (\ref{cameratoglobal}), therefore the errors of plane estimation in Fig. \ref{fig:est_exp-crop} and path following in Fig. \ref{fig:tracking_err_exp-crop} also include the uncertainty of the odometry.}

\begin{figure}[h]
\centering
\includegraphics[width=0.7\linewidth]{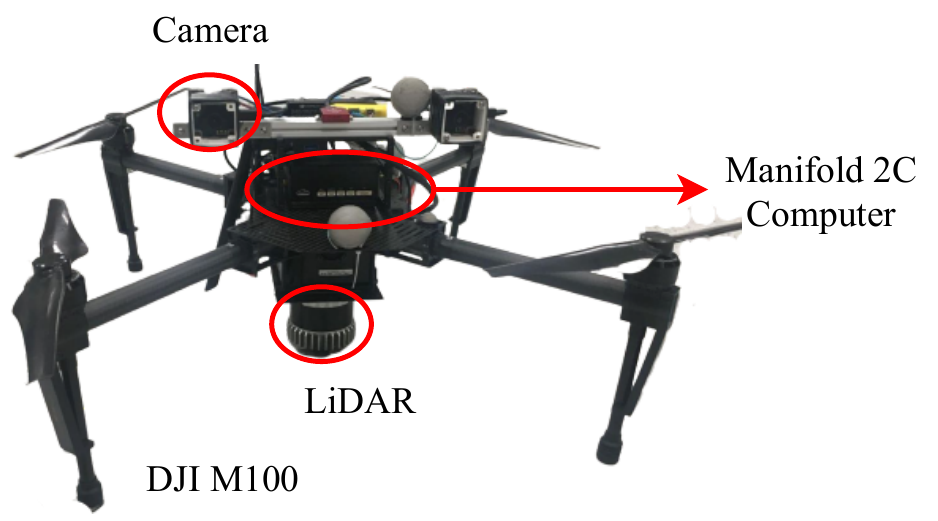}
\caption{The DJI M100 platform with on-board sensors (cameras, LiDAR) and computer.}
\label{fig:hardware}
\end{figure}

\begin{figure}[h]
\centering
\includegraphics[width=1\linewidth]{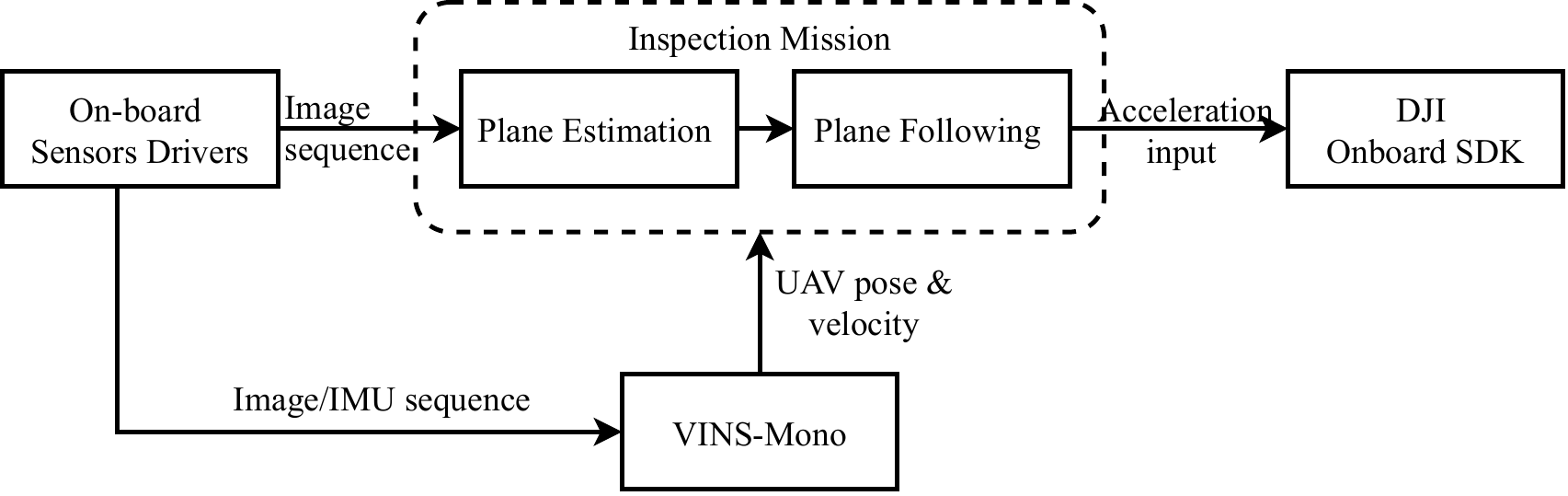}
\caption{The software \& data flow of the experiment platform for inspection task.}
\label{fig:softwareflow}
\end{figure}
\begin{figure}[h]
\centering
\includegraphics[width=0.85\linewidth]{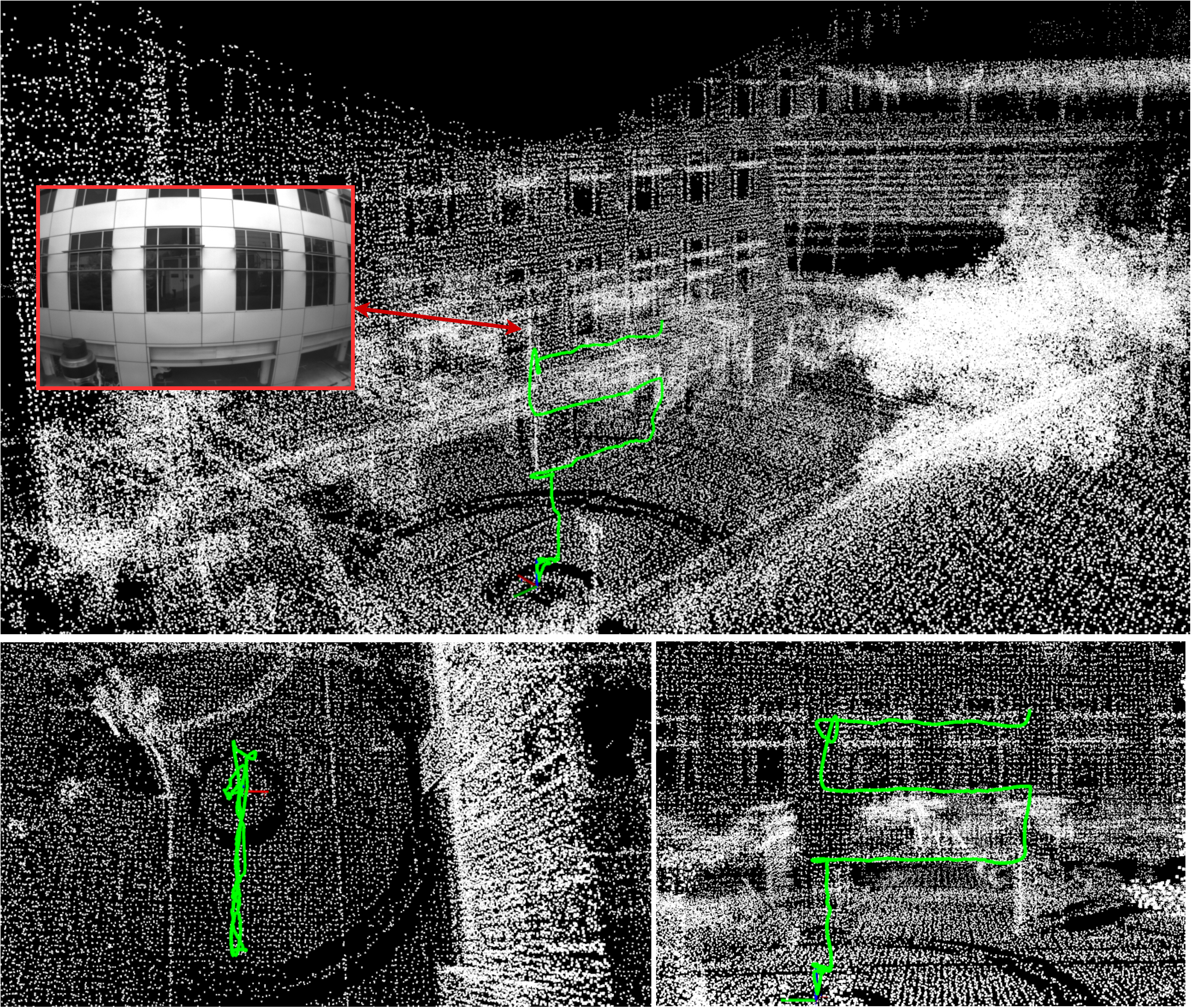}
\caption{The inspection trajectories (green curves) w.r.t. a building (white point cloud).}
\label{fig:scene}
\end{figure}
\begin{figure}[h]
	\centering
	\includegraphics[width=0.85\linewidth]{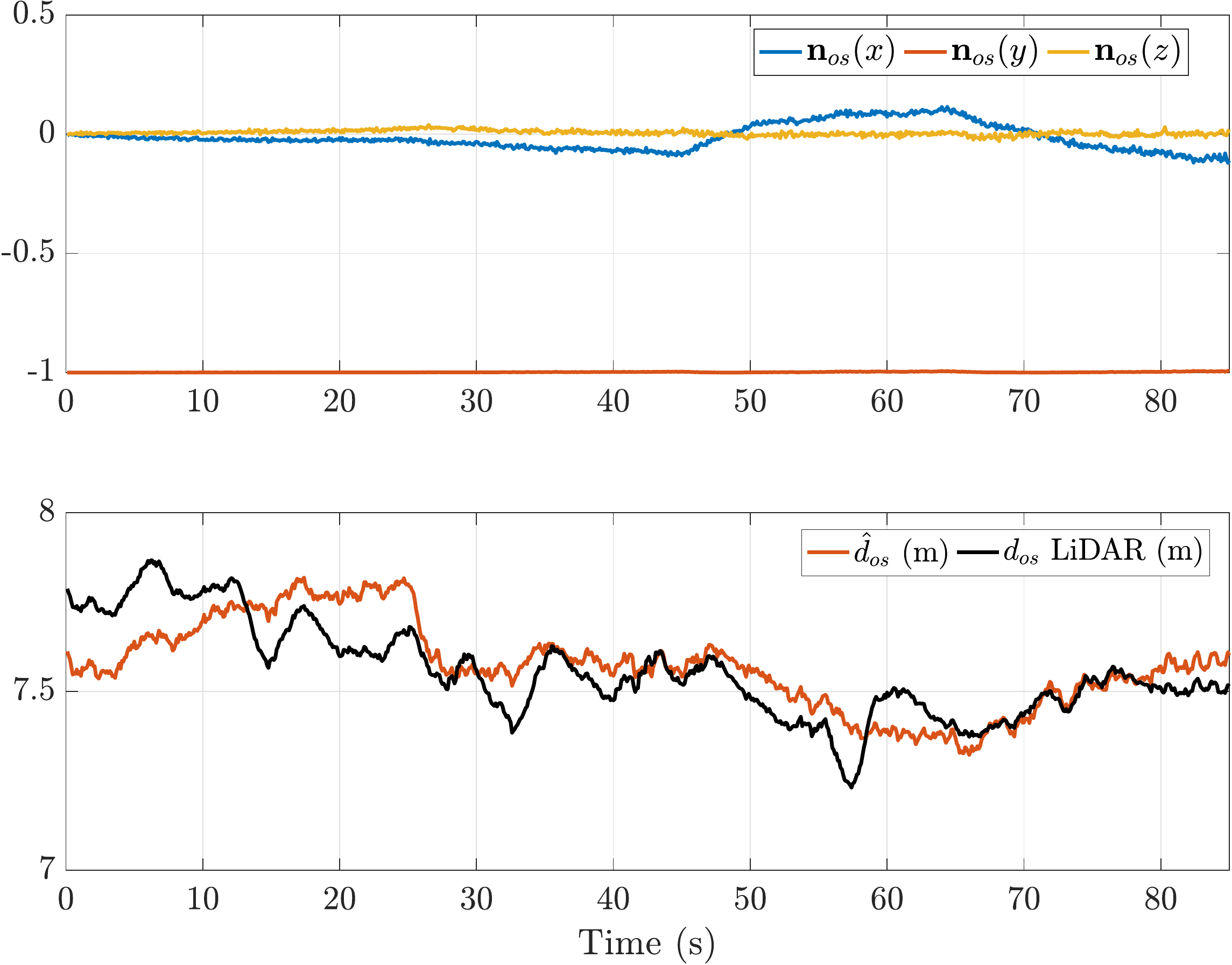}
	\caption{ The sampled plane parameters estimation $\hat\nbf_{os}$ and $\hat d_{os}$, compared to the separation distance provided by $d_{s}$ LiDAR.}
	\label{fig:est_exp-crop}
\end{figure}
\begin{figure}
	\centering
	\includegraphics[width=0.85\linewidth]{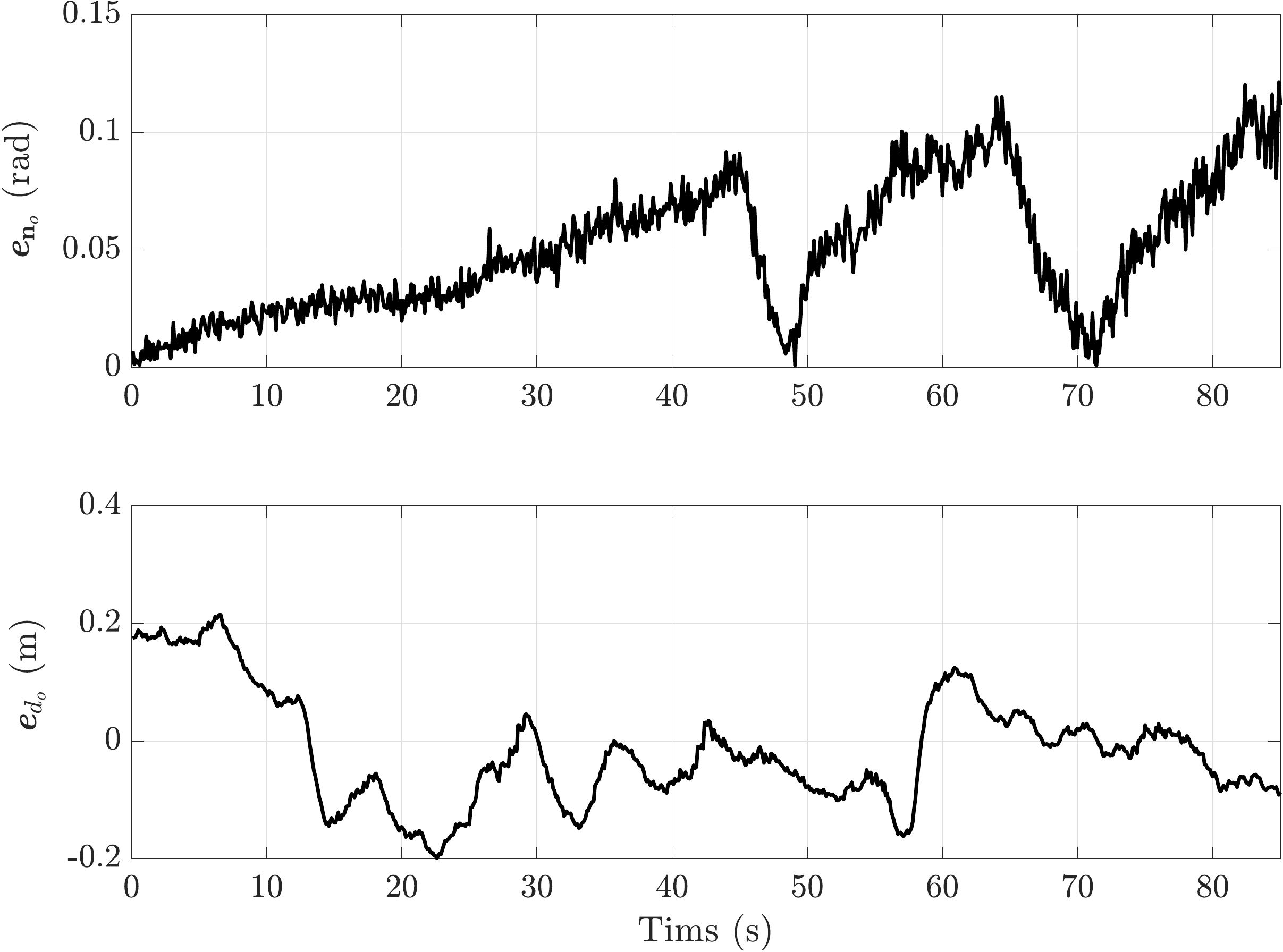}
	\caption{The plane estimation error, the normal vector estimation error $e_{\nbf_o}$ and the depth estimation error $e_{d_o}$.}
		\label{fig:est_err_exp-crop}
	\end{figure}
\begin{figure}[h]
\centering
\includegraphics[width=0.85\linewidth]{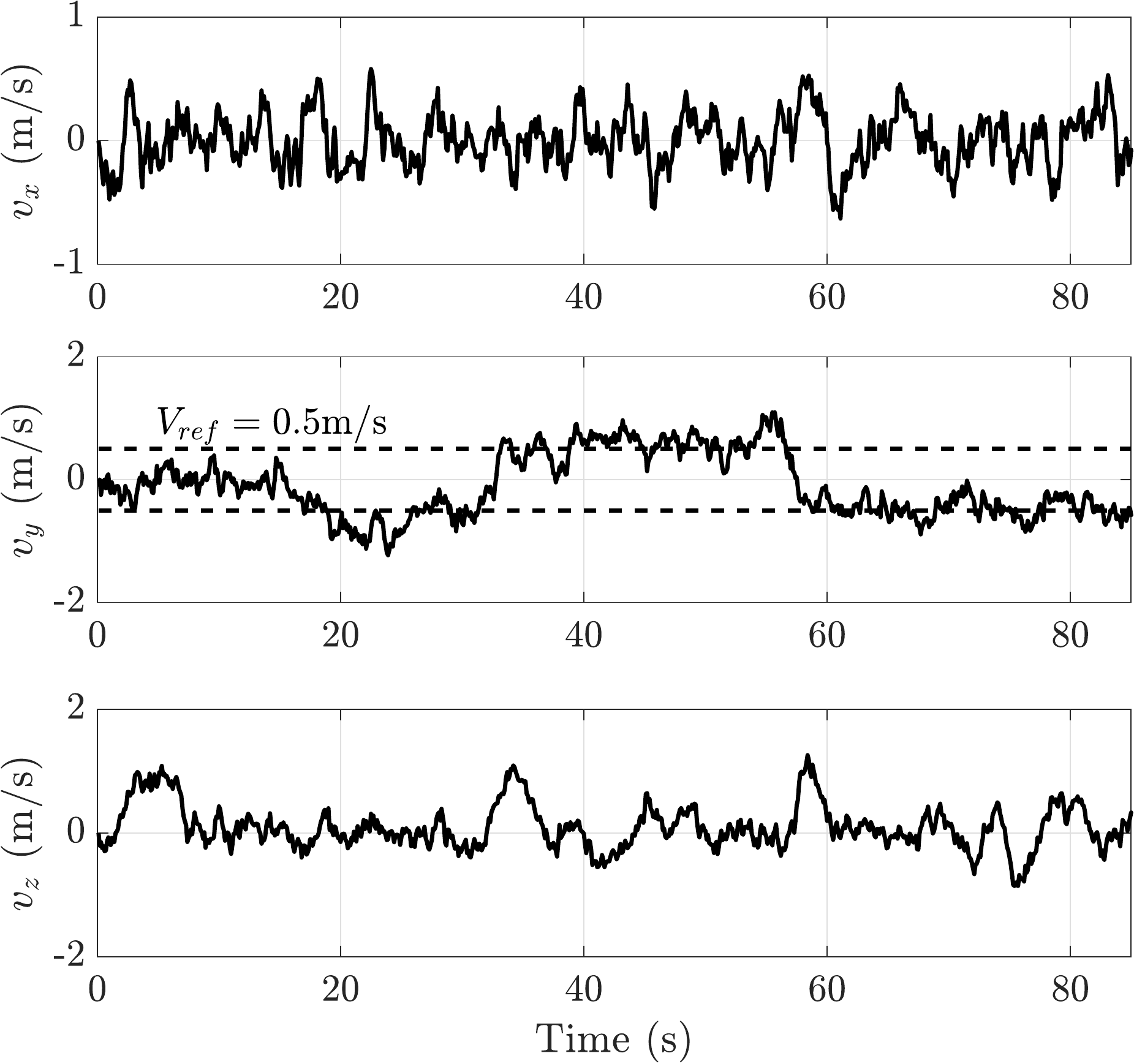}
\caption{The velocity of the inspection trajectory,with $v_r = 0.5$m/s.}
\label{fig:vel_exp-crop}
\end{figure}
\begin{figure}[h]
\centering
\includegraphics[width=0.85\linewidth]{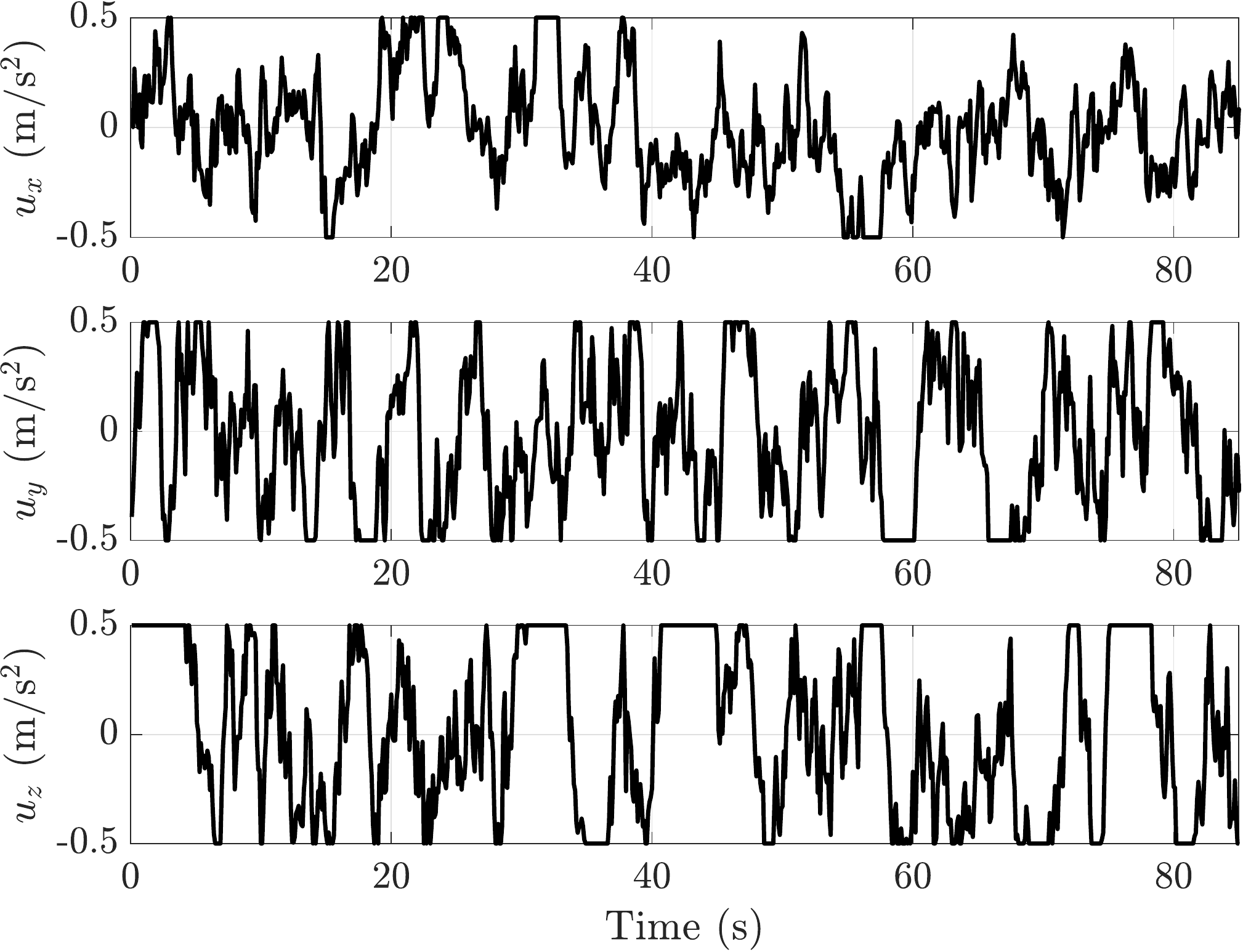}
\caption{The control $\ubf$ with constraints $\|\ubf\|_{\infty}\le 0.5$m/$s^2$. }
\label{fig:acc}
\end{figure}
\begin{figure}
\centering
\includegraphics[width=0.85\linewidth]{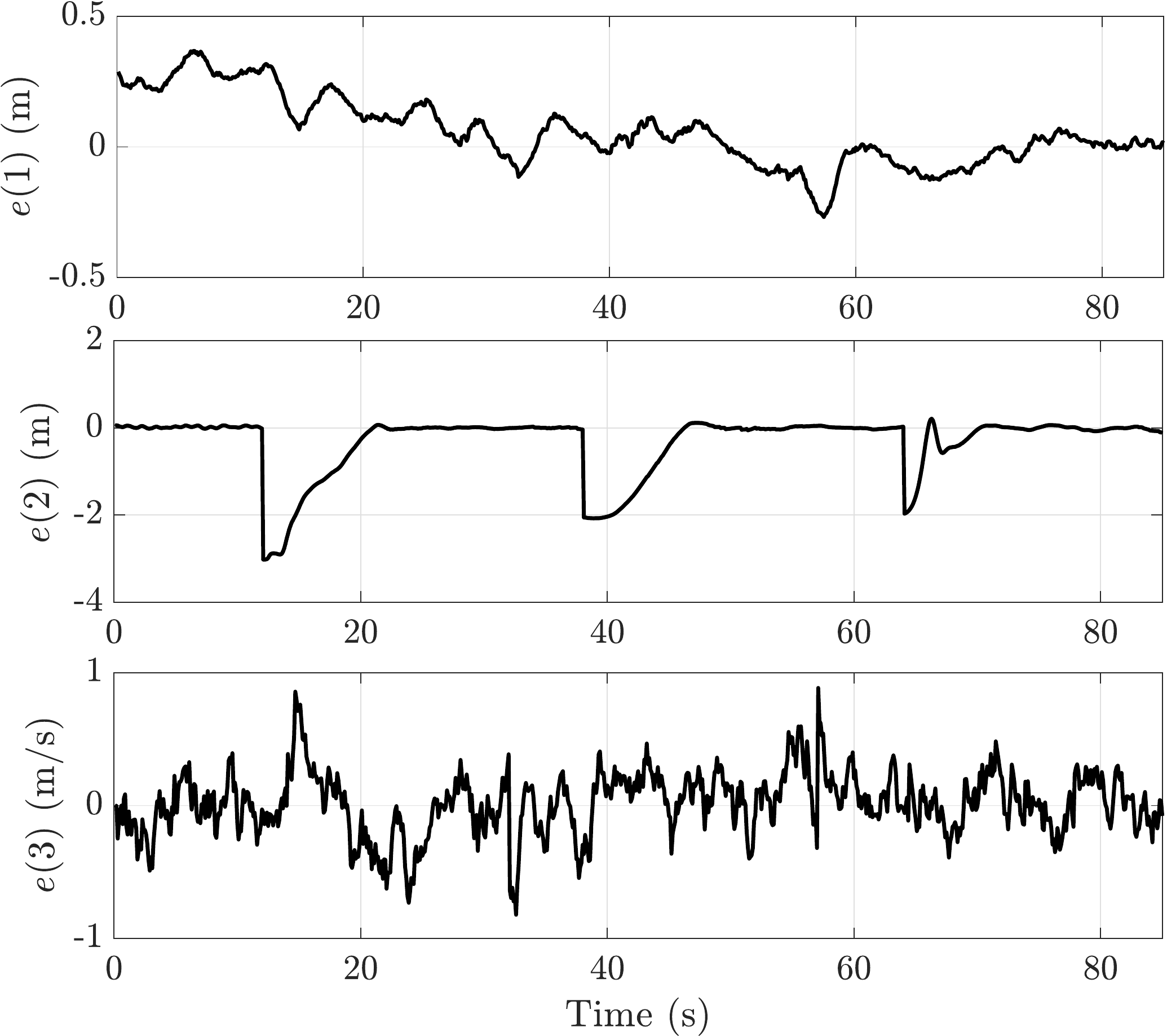}
\caption{The plane path following errors $\ebf$, include separation distance error $\ebf(1)$, round displacement error $\ebf(2)$, and velocity error $\ebf(3)$.}
\label{fig:tracking_err_exp-crop}
\end{figure}


	\section{Conclusion}
		In this paper, an inspection task, which is formulated as on-line plane estimation and plane following control, is considered. First, a vision based adaptive plane parameters observer is proposed, which can realize stable plane estimation under very mild conditions. Next, an MPC based plane follower that can automatically integrate the inspection references is designed. The feasibility and stability properties of the controller are also investigated. The effectiveness of the proposed plane estimation and following are verified and validated by simulations and experiments.

		In the future, the functional completeness will be further strengthened by integrating inspection boundary detection and collision avoidance so as to fulfill practical inspection tasks.
		\ifCLASSOPTIONcaptionsoff
		\newpage
		\fi

		
		
		%
		
		\bibliographystyle{IEEEtran}
		\bibliography{mybib}
		
		%
		
%
%
%
		
		
		

	\end{document}